\documentclass[sigconf]{acmart}

\usepackage{amsmath, amssymb}
\usepackage{amsfonts}
\AtBeginDocument{%
  }

\setcopyright{acmcopyright}
\copyrightyear{2024}
\acmYear{2024}
\acmDOI{XXXXXXX.XXXXXXX}

\acmPrice{15.00}
\acmISBN{978-1-4503-XXXX-X/18/06}

\usepackage{algorithm}
\usepackage{algorithmic}

\usepackage{makecell}
\usepackage{multirow}
\usepackage{threeparttable}
\usepackage{changepage}
\usepackage{tikz}
\usepackage{diagbox}
\usepackage{array}
\usepackage{stmaryrd}
\usepackage{subfigure}
\usepackage{multirow} 
\usepackage{xcolor}

\usepackage{graphicx}
\usepackage{color}
\usepackage{url}
\usepackage{bm}
\usepackage[all]{xy}
\usepackage{arydshln}
\usepackage{tikz}
\usepackage{float}

\usepackage{enumerate}
\usepackage{pifont}
\usepackage{soul}

\newtheorem{theorem}{Theorem}

\newtheorem{definition}{Definition}
\floatname{algorithm}{Algorithm}

\definecolor{GoogleRed}{RGB}{244, 66, 60}
\definecolor{GoogleGreen}{RGB}{10, 168, 88}
\definecolor{GoogleBlue}{RGB}{45, 133, 240}
\definecolor{GoogleYellow}{RGB}{255, 187, 50}
\definecolor{Gray}{RGB}{128,128,128}

\usepackage{comment}

\usepackage{pifont}

\usepackage{newfloat}
\usepackage{listings}

\begin{document}

\title{PAGE: Equilibrate Personalization and Generalization in Federated Learning}


\author {
	Qian Chen\textsuperscript{\rm 1},
	Zilong Wang\textsuperscript{\rm 1\;*},
	Jiaqi Hu\textsuperscript{\rm 1},
	Haonan Yan\textsuperscript{\rm 2},
	Jianying Zhou\textsuperscript{\rm 3},
	Xiaodong Lin\textsuperscript{\rm 4}    
}
\affiliation {
	\textsuperscript{\rm 1}
	\institution{State Key Laboratory of Integrated Service Networks, Xidian University, Xi'an, China}\city{}\country{}
}
\affiliation {
	\textsuperscript{\rm 2}
	\institution{Zhejiang Key Laboratory of Multi-dimensional Perception Technology Application and Cybersecurity, Zhejiang, China}\city{}\country{}
}
\affiliation {
	\textsuperscript{\rm 3}
	\institution{iTrust, Singapore University of Technology and Design, Singapore, Singapore}\city{}\country{}
}
\affiliation {
	\textsuperscript{\rm 4}
	\institution{School of Computer Science, University of Guelph, Guelph, Canada}\city{}\country{}
}
\email{
qchen\_4@stu.xidian.edu.cn, zlwang@xidian.edu.cn, {ivyhjq7, yanhaonan.sec}@gmail.com, jianying\_zhou@sutd.edu.sg,
}
\email
{ xlin08@uoguelph.ca
}
\thanks{\textsuperscript{\rm *} Corresponding author.}

%
%
%
%
%
%
%
%

\renewcommand{\shortauthors}{Chen et al.}

\begin{abstract}

Federated learning (FL) is becoming a major driving force behind machine learning as a service, where customers (clients) collaboratively benefit from shared local updates under the orchestration of the service provider (server). Representing clients' current demands and the server's future demand, local model personalization and global model generalization are separately investigated, as the ill-effects of {\it data heterogeneity} enforce the community to focus on one over the other. However, these two seemingly competing goals are of equal importance rather than black and white issues, and should be achieved simultaneously. In this paper, we propose the first algorithm to balance personalization and generalization on top of game theory, dubbed PAGE, which reshapes FL as a co-opetition game between clients and the server. To explore the equilibrium, PAGE further formulates the game as Markov decision processes, and leverages the reinforcement learning algorithm, which simplifies the solving complexity. Extensive experiments on four widespread datasets show that PAGE outperforms state-of-the-art FL baselines in terms of global and local prediction accuracy simultaneously, and the accuracy can be improved by up to 35.20\% and 39.91\%, respectively. In addition, biased variants of PAGE imply promising adaptiveness to demand shifts in practice.

\end{abstract}


\maketitle

\section{Introduction}
With the rapid proliferation of data constantly generated on pervasive mobile and Web-of-Things (WoT) devices, federated learning (FL) has emerged as a promising distributed machine learning (ML) paradigm that enables efficient data usage by unleashing the computation power on devices \cite{konevcny2016,Yang2023}. In typical FL (TFL) \cite{McMahan2017,Li2020FedProx,Acar2021FedDyn,Karimireddy2020SCAFFOLD,Chen2023Dap-FL}, represented by FedAvg \cite{McMahan2017}, a central server orchestrates a group of clients to train a single global model with desirable generalization by iteratively averaging local models rather than accessing raw data. Serving as a step towards high prediction accuracy and efficiency for ML-as-a-service (MLaaS), TFL is poised to revolutionize myriad WoT applications, such as the next word prediction on Google's Gboard on Android \cite{Bonawitz2019}, healthcare \cite{Liu2023}, and e-commerce \cite{Niu2020}, etc.

However, TFL suffers severely from the {\it data heterogeneity} \cite{Li2020} issue, which is a fundamental challenge attributed to non-independent identically distributed (Non-i.i.d.) local data. To be specific, the prediction accuracy of a single global model on individual clients is significantly reduced in the presence of heterogeneous local data distributions. For instance, clients from different demographics are likely to require totally different prediction results for the same sample due to diverse cultural nuances, while a single global model cannot generalize well in this case.

To overcome the ill-effects of {\it data heterogeneity}, personalized FL (PFL) has sparked increasing interest during the past few years, where customized local models are constructed for individual clients to provide satisfactory personalization \cite{Dinh2020pFedMe,Li2021Ditto,Singhal2021FedRECON,Zhang2023FedALA}. Currently, the research trend is to accommodate the generalized global model as personalized local models. In this case, global model generalization is inevitably sacrificed with the improvement of local model personalization \cite{Tan2022}. (An in-depth discussion of more related works is given in Section 2 and Appendix A.) But it is tempting to ask: Is the personalized local model in PFL, or perhaps the generalized global model in TFL, the most practical demand on earth? Although this pair of seemingly competing goals has enforced the FL community to focus on one over the other, it is never a black and white issue. Taking MLaaS as an example, customers require local models with desirable personalization, which is a current demand. On the contrary, generalized global models are pursued by the service provider to yield a better initialization to fine-tune local models for numerous new participants, which is referred to as the future demand. 

Recently, Chen {\it et al.} \cite{Chen2022FED-ROD} tried to draw attention back from personalization to their reconciliation, where the optimization priority between personalization and generalization was eliminated. Besides, the widely used regularizer was proven less effective and hence removed. Still, a definite insight into the equality between personalization and generalization was not claimed. To extend their insight, we specify that personalization and generalization share equal status in FL, and the balance between them is much needed. Back to the MLaaS scenario, balance refers to a moderate condition satisfying current and future demands simultaneously. Yet, an intuitive question springs to mind:
\begin{itemize}
	\item[]{\it How to achieve the balance between local model personalization and global model generalization in FL?}
\end{itemize}

In response, we propose a \underline{p}ersonalization \underline{a}nd \underline{g}eneralization \underline{e}quilibrium (PAGE) FL algorithm. Following the optimization problem in \cite{Chen2022FED-ROD}, we formulate FL as a joint evolution with mutual restraints between global and local models by removing the regularizer. Such an evolution is intractable, as the optimization of competing objectives would get out of control with the removal of the regularizer. Intuitively, the iterative evolution can be viewed as a co-opetition game, where the personas of clients and the server switch to players with leader-follower relations. As a result, to balance the competing objectives, PAGE establishes an implicit relation between local and global models through a feedback multi-stage multi-leader single-follower (MLSF) Stackelberg game \cite{GameTheory1998}. Additionally, to simplify the exploration of the game equilibrium, i,e., balance, PAGE further re-formulates the game as Markov decision processes (MDPs) \cite{Bellman1957}, and leverages the deep deterministic policy gradient (DDPG) \cite{Lillicrap2016} algorithm. 

The main {\bf contributions} are summarized as follows:

\begin{itemize}
	\item{To the best of our knowledge, PAGE is the first algorithm to balance generalization and personalization in FL. In particular, PAGE establishes the relation between personalization and generalization on top of game theory.}
	\item{We re-formulate the game as server-level and client-level MDPs, and explore the equilibrium by reinforcement learning (RL). Through rigorous analysis, the existence of the equilibrium is proved.}
	\item{We evaluate PAGE on four widespread databases. Experimental results show that PAGE outperforms the state-of-the-art (SOTA) PFL and TFL in terms of global and local prediction accuracy simultaneously, and the accuracy can be improved by up to 35.20\% and 39.91\%, respectively. Besides, biased variants of PAGE imply promising adaptiveness to varying demand shifts in practice.}
\end{itemize}

\section{Related Work}
Since the birth of FL, {\it data heterogeneity} has been a root cause of the tension between generalization and personalization. Accordingly, the research community has been divided into TFL \cite{McMahan2017,Li2020FedProx,Acar2021FedDyn,Karimireddy2020SCAFFOLD,Chen2023Dap-FL} and PTL \cite{Dinh2020pFedMe,Li2021Ditto,Singhal2021FedRECON,Zhang2023FedALA,Chen2022FED-ROD}, focusing on global model generalization and local model personalization, respectively. Below we discuss the SOTA baselines most relevant to PAGE, and a more comprehensive literature review can be found in Appendix A.

\textbf{Typical federated learning.} Solutions for {\it data heterogeneity} stemmed from FedAvg \cite{McMahan2017}, which was a standard and fundamental algorithm. Shortly, it was proven hard to meet Non-i.i.d. data \cite{Li2020}. Later on, to mitigate this issue, Li {\it et al.} \cite{Li2020FedProx} proposed FedProx to generalize FedAvg by adding a proximal term to the objective, which improved the stability facing heterogeneous data. Similarly, FedDyn investigated linear and quadratic penalty terms \cite{Acar2021FedDyn}. Different from these regularization methods, SCAFFOLD corrected local updates through variance reduction \cite{Karimireddy2020SCAFFOLD}. Recently, Chen {\it et al.} \cite{Chen2023Dap-FL} proposed Dap-FL to adaptively control local contributions for aggregation. Although the above TFL algorithms could yield expected global model generalization, the single global model setting struggled to satisfy customers' current demands in MLaaS, i.e., local model personalization.

\textbf{Personalized federated learning.} To overcome {\it data heterogeneity}, PFL has drawn significant research interest in training customized models adapting to diverse local data. For instance, pFedMe optimized a bi-level problem using regularized local loss functions with $ L_{2} $-norm, where personalized local models were decoupled from the global model optimization \cite{Dinh2020pFedMe}. Ditto conducted a similar regularization method, but differed by switching the priority between global and local objectives \cite{Li2021Ditto}. Besides, Singhal {\it et al.} leveraged model-agnostic meta-learning to fine-tune local models \cite{Singhal2021FedRECON}. Most recently, Zhang {\it et al.} \cite{Zhang2023FedALA} proposed FedALA, which adaptively aggregated the downloaded global model and local models towards local objectives at the element level. However, PFL fared less well in global model generalization, which cannot meet the future demand of service providers in practice. One closely relevant work was FED-ROD, where an implicit regularizer was introduced to consider generalization in the presence of personalization \cite{Chen2022FED-ROD}. Although FED-ROD decoupled local and global models, a definite insight into their equal statuses was absent.

To the best of our knowledge, no prior arts take the balance of local model personalization and global model generalization into account, while PAGE bridges this gap through game theory, thereby satisfying current and future demands simultaneously. More importantly, PAGE converts sub-games into MDPs, and derives the equilibrium by adaptively adjusting local training hyper-parameters and aggregation weights on top of RL.

\section{Problem Statement}
In this section, we first formalize TFL and PFL systems, then identify the problem to be solved in this paper\footnote{For clarity, we summarize important notations in Appendix B.}. Generally, FL involves $ N $ clients $ \mathbb{C}\!=\!\left\{c_{i},i\!=\!1,\!\cdots\!,N \right\} $ and a central server $\mathit{CS}$. Each $ c_{i} $ has a private local dataset $ D_{i}\!=\!\left\{\left(x_{i,k},y_{i,k}\right), k\!=\!1,\!\cdots\!,\vert D_{i}\vert\right\} $, where $ \vert D_{i}\vert $ is the data size, $ x_{i,k} $ is the feature of a specific sample, and $ y_{i,k} $ is the corresponding label. Also, $ \mathit{CS} $ owns a public dataset $ D_{\mathit{CS}} $. The goal of TFL and PFL is to collaboratively train global and local models, respectively. Supposing $ f_{i}(x_{i,k},y_{i,k};w_{i}) $ denotes $ c_{i} $'s local loss function (simply expressed as $ f_{i}(w_{i}) $), the global loss function is denoted by $ F(\cdot) $ and defined as:
\begin{equation}\label{eq-1}
	F(W)=\sum_{i=1}^{N}\left(p_{i}\!\cdot\!\mathbb{E}_{D_{i}}\left[f_{i}(x_{i,k},y_{i,k};w_{i})\right]\right)=\sum_{i=1}^{N}\left(p_{i}\!\cdot\! f_{i}(w_{i})\right) ,
\end{equation}
where $ w_{i} $ is $ c_{i} $'s local model, $ W $ is the global model, $ p_{i}\!\in\!(0,1) $ is the aggregation weight, and $ \sum_{i=1}^{N}p_{i}\!=\!1 $. 

Mathematically, TFL aims to train a single global model with promising generalization, shown as:
\begin{equation}\label{eq-2}
	W^{*}=\mathop{\arg\min}_{W} F\left(D_{1},\!\cdots\!,D_{N};W\right),
\end{equation}
where $ W^{*} $ is the converged global model. At the opposite end of the spectrum, to tackle data heterogeneity issues, PFL customizes local models with satisfactory personalization, formally given as:
\begin{equation}\label{eq-3}
	\left\{ \begin{aligned}
		&W^{*}=\mathop{\arg\min}\limits_{W}\left\{ F\left(W\right):=\sum\nolimits_{i=1}^{N}(p_{i}\!\cdot\!( f_{i}(w_{i})+\mathcal{R}_{i}))\right\}, \\ 
		&~\text{s.t.}~~~w_{i}^{*}=\mathop{\arg\min}\limits_{w_{i}}\left\{f_{i}(w_{i})+\mathcal{R}_{i}\right\}, i=1,\!\cdots\!,N,\\ 
	\end{aligned} \right.
\end{equation}
where $ w_{i}^{*} $ is the optimal local model, and the regularizer $ \mathcal{R}_{i} $ controls the strength of $ W $ to $ w_{i} $.

Different from TFL and PFL, we concentrate on balancing global model generalization and local model personalization, rather than facilitating any of them to a position of prominence. Following \cite{Chen2022FED-ROD}, we define an optimization problem: 
\begin{equation*}
	\textbf{P}_{\bm 0}:\left\{ \begin{aligned}
		&W^{*}=\mathop{\arg\min}\limits_{W}\left\{ F\left(W\right):=\sum\nolimits_{i=1}^{N}\left(p_{i}\!\cdot\! f_{i}(w_{i})\right)\right\}, \\ 
		&w_{i}^{*}=\mathop{\arg\min}\limits_{w_{i}}\left\{f_{i}(w_{i})\right\}, i=1,\!\cdots\!,N.\\ 
	\end{aligned} \right.
\end{equation*}
In this case, $ \mathit{CS} $ and $ c_{i} $ would conduct an iterative co-opetition, aiming at a joint evolution with mutual restraints between $ W $ and $ w_{i} $. Specifically, in any given round $ t\!=\!1,\!\cdots\!,T $, each $ c_{i} $ initializes the local model $ w_{i}(t) $ as the most recent global model $ W(t) $ received from $ \mathit{CS} $. Then, $ c_{i} $ updates $ w_{i}(t) $ for $ \alpha_{i}(t) $ epochs, expressed as:
\begin{equation}\label{eq-4}
	\hat{w}_{i}(t)=\mathit{Train}\left(\eta_{i}(t),\alpha_{i}(t);w_{i}(t)\right),
\end{equation}
where $ \hat{w}_{i}(t) $ is the updated local model, and $ \eta_{i}(t) $ is the learning rate. Subsequently, each $ c_{i} $ uploads $ \hat{w}_{i}(t) $ to $ \mathit{CS} $, and $ \mathit{CS} $ assigns $ p_{i}(t) $ for every $ \hat{w}_{i}(t) $ to update the global model by aggregation, shown as:
\begin{equation}\label{eq-5}
	W(t+1)=\sum\nolimits_{i=1}^{N}\left(p_{i}(t)\!\cdot\!\hat{w}_{i}(t)\right).
\end{equation}

\section{Proposed Method: PAGE}

\subsection{Game (Relation) Establishment}

To control the delicate balance in $ \textbf{P}_{\bm 0} $, it is essential to establish a more effective relation between $ W $ and $ w_{i} $. In general, the balance-controlling factors are equivalent to the counterparts impacting $ f_{i}(\cdot) $ and $ F(\cdot) $. Empirical results show that the most significant factors are $ \alpha_{i}(t) $, $ \eta_{i}(t) $, and $ p_{i}(t) $ \cite{Chen2023Dap-FL,Wu2021}. Concretely, a larger (smaller) $ \alpha_{i}(t) $ provides more (fewer) steps of the optimization of $ f_{i}(\cdot) $, thereby contributing more (lesser) to local model fitness over $ D_{i} $, i.e., local model personalization. $ \eta_{i}(t) $ wields the influence in a similar way. Besides, $ \alpha_{i}(t) $ and $ \eta_{i}(t) $ impact $ F(\cdot) $ in an indirect manner, where $ f_{i}(\cdot) $ plays a role in a bridge. Loosely speaking, over-optimized $ f_{i}(\cdot) $ derived from larger $ \alpha_{i}(t) $ and/or $ \eta_{i}(t) $ holds down the convergence of $ F(\cdot) $ to some extent, i.e., excessive local model personalization deteriorates global model generalization. Yet, appropriate $ p_{i}(t) $ could mitigate the bias of over-optimized $ f_{i}(\cdot) $ to facilitate the convergence of $ F(\cdot) $, which, in turn, drags $ f_{i}(\cdot) $ from overfitting. More critically, the influence of these balance-controlling factors on either personalization or generalization might even go beyond the apparently positive or negative correlation in practice, which exacerbates the complexity of the relation establishment.

From a game theory point of view, the iterative evolution between $ w_{i}(t) $ and $ W(t) $ subject to balance-controlling factors can be regarded as a multi-stage co-opetition game between clients and $ \mathit{CS} $ with leader-follower sequences, where leaders move ahead of the follower in each stage. On this ground, we re-formulate $ \textbf{P}_{\bm 0} $ as a feedback multi-stage MLSF Stackelberg game in Definition \ref{def-1}, based on which an implicit relation between $ W $ and $ w_{i} $ is established.
\begin{definition}\label{def-1}
	$ \textbf{\rm P}_{\bm 0} $ can be formulated as a feedback multi-stage MLSF Stackelberg game, defined as:
	\begin{equation*}
		\begin{aligned}
			\textbf{\rm P}^{\bm\prime}_{\bm 0}=\Big\llbracket\big\langle\! &\left\{c_{i}\right\}_{i=1}^{N}\in\mathbb{C},\!\mathit{CS}\big\rangle,\big\langle\!\left\{g_{i}(t)\in\mathcal{G}_{i}\right\}_{i=1}^{N}\!,g_{\mathit{CS}}(t)\in\mathcal{G}_\mathit{CS}\big\rangle,\\ 
			&\big\langle\!\left\{u_{i}(t)\right\}_{i=1}^{N}\!,u_\mathit{CS}(t) \big \rangle,z(t)\in\mathcal{Z},t=1,\!\cdots\!,T~\Big\rrbracket, \text{where}
		\end{aligned}
	\end{equation*}
	\begin{itemize}
		\item{$ c_{i},i=1,\!\cdots\!,N $ are leaders, and $ \mathit{CS} $ is the follower.}
		\item{$ t\!=\!1,\!\cdots\!,T $ represents the stage of the game. Note that the initial global model distribution is not involved in $ \textbf{\rm P}^{\bm\prime}_{\bm 0} $.}
		\item{$ g_{i}(t)\!=\!\left[\alpha_{i}(t),\eta_{i}(t)\right] $ is $ c_{i} $'s strategy in the $ t $-th stage, and $ \mathcal{G}_{i} $ is the strategy space.}
		\item{$ g_{\mathit{CS}}(t)\!=\!\left[p_{1}(t),\!\cdots\!,p_{N}(t)\right] $ is $ \mathit{CS} $'s reacting strategy to all $ g_{i}(t),i\!=\!1,\!\cdots\!,N $, and $ \mathcal{G}_{\mathit{CS}} $ is the strategy space.}
		\item{$ u_{i}(t)=1/f_{i}\left(\hat{w}_{i}(t)\right) $ is $ c_{i} $'s utility function.}
		\item{$ u_{\mathit{CS}}(t)=1/F\left(W(t\!+\!1)\right)=1/\sum_{i=1}^{N}(p_{i}(t)\cdot f_{i}(\hat{w}_{i}(t))) $ is $ \mathit{CS} $'s utility function.}
		\item{$ z(t) $ is the gaming condition, and $ \mathcal{Z} $ is the condition space.}
	\end{itemize}
\end{definition} 

Definition \ref{def-1} depicts dynamic conflict situations between clients and $ \mathit{CS} $ over time, in which each $ c_{i} $ operates $ g_{i}(t) $, and $\mathit{CS}$ optimizes $ g_{\mathit{CS}}(t) $ subject to the constraints of all clients' strategies in each stage. Also, clients are able to infer $ \mathit{CS} $’s reaction to any strategies they operate. Therefore, each $ c_{i} $ could operate a strategy that maximizes the utility, given the predicted behavior of $ \mathit{CS} $. 

Notably, the equilibrium of the game $ \textbf{P}^{\bm\prime}_{\bm 0} $ provides a terminating condition for the pursuing balance in $ \textbf{P}_{\bm 0} $, whose existence is confirmed at the end of this section (Theorem \ref{thm-1}). Next, in line with the general equilibrium solving method in Stackelberg games \cite{GameTheory1998}, we split $ \textbf{P}^{\bm\prime}_{\bm 0} $ as the Server-level and Client-level sub-games to explore the appropriate strategy sequences in the equilibrium separately.

\subsection{Strategy Exploration in the Server-level Sub-game}

For the server-level sub-game, the equilibrium of $ \textbf{P}^{\bm\prime}_{\bm 0} $ indicates the optimal strategy sequence of $ \mathit{CS} $, where the strategy in the current stage hinges on the gaming result in the previous stage and impacts next-stage strategies. However, the optimal strategy sequence is intractable through general backward induction algorithms \cite{GameTheory1998}, as the complexity increases exponentially with $ t $.

Intuitively, such an over-time strategy conducting process is equivalent to an MDP \cite{Bellman1957}, where $ \mathit{CS} $ makes decisions about $ p_{i}(t) $ sequentially through interacting with the environment, i.e., evaluating local updates. In other words, the MDP 3-tuple could be naturally found in the server-level sub-game, and the optimal strategy sequence could be solved by RL algorithms. Therefore, we first model the Server-level sub-game as an MDP $ \big\langle \mathcal{S}_{\mathit{CS}}, \mathcal{A}_{\mathit{CS}}, R_{\mathit{CS}}(\cdot) \big\rangle $, where $ \mathcal{S}_{\mathit{CS}} $ is the state space, $ \mathcal{A}_{\mathit{CS}}\equiv\mathcal{G}_{\mathit{CS}} $ is the action space, and $ R_{\mathit{CS}}(\cdot) $ is the reward function. Below we define the 3-tuple in detail.

\noindent$ \bullet $ \textit{State:} $ s_{\mathit{CS}}(t)\triangleq\left[\hat{\mathit{acc}}_{1}(t),\!\cdots\!,\hat{\mathit{acc}}_{N}(t)\right]\in\mathcal{S}_{\mathit{CS}} $, where $ \hat{\mathit{acc}}_{i}(t) $ is the prediction accuracy of $ \hat{w}_{i}(t) $ on $ D_{\mathit{CS}} $.

\noindent$ \bullet $ \textit{Action:} $ a_{\mathit{CS}}(t)\triangleq g_{\mathit{CS}}(t)= \left[p_{1}(t),\!\cdots\!,p_{N}(t)\right]\in\mathcal{A}_{\mathit{CS}} $.

\noindent$ \bullet $ \textit{Reward:} $ r\!_{\mathit{CS}}(t)=R_{\mathit{CS}}(s_{\mathit{CS}}(t),a_{\mathit{CS}}(t),s_{\mathit{CS}}(t+1))\triangleq u_{\mathit{CS}}(t)=1/\sum\nolimits_{i=1}^{N}\left(p_{i}(t)\!\cdot\! f_{i}(\hat{w}_{i}(t))\right) $.

Mathematically, the Server-level MDP is defined as:
\begin{equation*}
	\textbf{P}^{\bm\prime}_{\bm 0}\_\mathit{CS}~:~\underset{{\mu}_{\mathit{CS}}(\cdot)}{\mathop{\max}}~J_{\mathit{CS}}(\cdot),
\end{equation*}
where $ \mu_{\mathit{CS}}(\cdot):s_{\mathit{CS}}(t)\rightarrow a_{\mathit{CS}}(t) $ is the policy, $ J_{\mathit{CS}}(\cdot)= \sum_{t=1}^{T}({\gamma}^{t-1}\!\cdot\! r_{\mathit{CS}}(t)) $, and $ \gamma $ is the discount factor. Note that $ \textbf{P}^{\bm\prime}_{\bm 0}\_\mathit{CS} $ is approximately equivalent to the Server-level sub-game, as $ \gamma $ is usually set as $ 0.99 $ in practice.

Due to high-dimensional and continuous action and state space, we introduce DDPG, which consists of a MainNet and a TargetNet with the same {\it Actor-Critic} structure \cite{Lillicrap2016}, to solve $ \textbf{P}^{\bm\prime}_{\bm 0}\_\mathit{CS} $. In the MainNet, the {\it Actor} is expressed as $ \mu_{\mathit{CS}}\left(\cdot;\theta_{\mathit{CS}}^{\mu}(t)\right) $, which takes $ s_{\mathit{CS}}(t) $ as the input and outputs $ a_{\mathit{CS}}(t) $ through the parameterized policy $ \theta_{\mathit{CS}}^{\mu}(t) $. The {\it Critic} takes $ s_{\mathit{CS}}(t) $ and $ a_{\mathit{CS}}(t) $ as the input and outputs the value of the parameterized state-action function $ Q^{\mu}_{\mathit{CS}}(\cdot;\theta^{Q}_{\mathit{CS}}(t)) $. In addition, the TargetNet is a copy of the MainNet, which is parameterized by $ \mu'_{\mathit{CS}} (\cdot; {\theta}^{\mu'}_{\mathit{CS}}(t)) $ and $Q^{\mu'}_{\mathit{CS}}(\cdot;{\theta}^{Q'}_{\mathit{CS}}(t) ) $. The detailed strategy exploration process is shown in Algorithm \ref{alg-1}, where the best policy $\mu_{\mathit{CS}}^{*}(\cdot)$ outputs the selected action sequence $ A^{*}_{\mathit{CS}}=\left[a^{*}_{\mathit{CS}}(1),\!\cdots\!,a^{*}_{\mathit{CS}}(T)\right] $, which is $ \mathit{CS} $'s optimal strategy sequence $ G^{*}_{\mathit{CS}}(1)\!=\!\left[g^{*}_{\mathit{CS}}(1),\!\cdots\!,g^{*}_{\mathit{CS}}(T)\right]\!\equiv\! A^{*}_{\mathit{CS}} $ in the equilibrium of $ \textbf{P}^{\bm\prime}_{\bm 0} $. 

\begin{algorithm}
	\small
	\caption{Global Aggregation Weights Tuning}  
	\begin{algorithmic}[1]\label{alg-1}
		\REQUIRE $ l_{CS}^{Cri} $ and $ l_{CS}^{Act} $ are the learning rates for {\it Critic} and {\it Actor} in the MainNet; $ \beta_{CS} $ is a tiny updating rate for the TargetNet; $ \vert B\vert $ is the batch size.
		\ENSURE $ p_{i}(t)\vert i=1,\!\cdots\!,N,t=1,\!\cdots\!,T $.
		\STATE Initialize $ \theta_{CS}^{\mu}(\cdot) $, $ \theta_{CS}^{Q}(\cdot) $, $ \theta_{CS}^{\mu'}(\cdot) $, and $ \theta_{CS}^{Q'}(\cdot) $;
		\FOR{$ t=1,\!\cdots\!,T $}
		\STATE Observe $ s_{CS}(t) $, and hence calculate $ r_{CS}(t) $;
		\STATE Randomly sample a batch of experience tuples \\$ \left(s_{\mathit{CS}}(\xi),\!a_{\mathit{CS}}(\xi),r_{\mathit{CS}}(\xi),s_{\mathit{CS}}(\xi\!+\!1) \right), \xi\!=\!1,\!\cdots\!,\vert B\vert  $;
		\FOR{$ \xi=1,\!\cdots\!,\vert B\vert $}
		\STATE Calculate $ {y}_{\mathit{CS}}(\xi)\!=\!{r}_{\mathit{CS}}(\xi)+\gamma\cdot{Q}_{\mathit{CS}}^{\mu'}({s}_{\mathit{CS}}(\xi+1),\mu'_{i}(s_{\mathit{CS}}(\xi+1);{\theta}^{\mu'}_{\mathit{CS}}(t\!-\!1));{\theta}_{\mathit{CS}}^{{Q}^{\prime}}(t\!-\!1)) $;
		\ENDFOR
		\STATE Calculate $ \mathit{Loss}_{\mathit{CS}}(t\!-\!1)\!=\!1/|B|\sum\nolimits_{\xi=1}^{|B|}(y_{\mathit{CS}}(\xi)-Q_{\mathit{CS}}^{\mu}(s_{\mathit{CS}}(\xi),a_{\mathit{CS}}(\xi);{\theta}_{\mathit{CS}}^{Q}(t\!-\!1) ) )^{2}  $;
		\STATE Update $ \theta_{\mathit{CS}}^{Q}\!(t), \theta_{\mathit{CS}}^{\mu}\!(t),\theta_{\mathit{CS}}^{\mu'}\!(t) $, and $ \theta_{\mathit{CS}}^{Q'}\!(t) $ as follows:\\
		$ \theta_{\mathit{CS}}^{Q}(t)\!=\!\theta_{\mathit{CS}}^{Q}(t\!-\!1)-{l}_{\mathit{CS}}^{Cri}\cdot {\nabla_{\theta_{\mathit{CS}}^{Q}}} \mathit{Loss}_{\mathit{CS}}(t\!-\!1) $,\\
		$ \theta_{\mathit{CS}}^{\mu}(t)\!=\!\theta_{\mathit{CS}}^{\mu}(t\!-\!1)+{l}_{\mathit{CS}}^{Act}\cdot {\nabla_{\theta_{\mathit{CS}}^{\mu}}} J_{\mathit{CS}}(t\!-\!1) $,\\
		$ \theta_{\mathit{CS}}^{\mu'}(t)\!=\!\beta_{\mathit{CS}} \cdot {\theta_{\mathit{CS}}^{\mu }(t\!-\!1)}+\left( 1\!-\!\beta_{\mathit{CS}} \right) \cdot \theta_{\mathit{CS}}^{\mu'}(t\!-\!1) $, \\ 
		$ \theta_{\mathit{CS}}^{Q'}(t)\!=\!\beta_{\mathit{CS}} \cdot {\theta_{\mathit{CS}}^{Q}(t\!-\!1)}+\left( 1\!-\!\beta_{\mathit{CS}} \right) \cdot {\theta }_{\mathit{CS}}^{Q'}(t\!-\!1) $;
		
		\ENDFOR
		\RETURN $ \mu_{\mathit{CS}}^{*}(\cdot) $;
		\RETURN $ G^{*}_{\mathit{CS}}(1)\!=\!A^{*}_{\mathit{CS}}\!=\!\left[p_{i}(t)\vert i\!=\!1,\!\cdots\!,N,t\!=\!1,\!\cdots\!,T\right] $.
	\end{algorithmic}
\end{algorithm}

\subsection{Strategy Exploration in the Client-level Sub-game}

In the same vein, we model $ c_{i} $'s Client-level sub-game as an MDP, and define the 3-tuple as follows.

\noindent$ \bullet $ \textit{State:} $ 
s_{i}(t)\triangleq\left[\mathit{acc}_{i}(t)\right]\in\mathcal{S}_{i} $, where $ \mathit{acc}_{i}(t) $ is the prediction accuracy of $ w_{i}(t) $ on $ D_{i} $, and $ \mathcal{S}_{i} $ is the state space.

\noindent$ \bullet $ \textit{Action:} $
a_{i}(t)\triangleq g_{i}(t)= \left[\alpha_{i}(t),\eta_{i}(t)\right]\in \mathcal{A}_{i} $, where $ \mathcal{A}_{i}\equiv\mathcal{G}_{i} $ is the action space.

\noindent$ \bullet $ \textit{Reward:} $ r_{i}(t)\!=\!R_{i}(s_{i}(t),a_{i}(t),s_{i}(t\!+\!1))\!\triangleq\!u_{i}(t)=\frac{1}{f_{i}(\hat{w}_{i}(t))} $.

Accordingly, $ c_{i} $'s Client-level MDP is defined as:
\begin{equation*}
	\textbf{P}^{\bm\prime}_{\bm 0}\_c_{i}~:~\underset{{\mu}_{i}(\cdot)}{\mathop{\max}}~J_{i}(\cdot),
\end{equation*}
where $ \mu_{i}(\cdot):s_{i}(t)\rightarrow a_{i}(t) $ is the policy, and $ J_{i}(\cdot)= \sum_{t=1}^{T}\left(\gamma^{t\!-\!1}\cdot{r}_{i}(t)\right) $. 

Similarly, $ \textbf{P}^{\bm\prime}_{\bm 0}\_c_{i} $ can be solved by performing Algorithm \ref{alg-2} along with the gaming process, which outputs the appropriate action sequence $ A^{*}_{i}=\left[a^{*}_{i}(1),\!\cdots\!,a^{*}_{i}(T)\right] $, i.e., the strategy sequence $ G^{*}_{i}(1)=\left[g^{*}_{i}(1),\!\cdots\!,g^{*}_{i}(T)\right]\equiv A^{*}_{i} $ in the equilibrium. 

\begin{algorithm}[htbp]
	\small
	\caption{Local Training Hyper-parameters Tuning}  
	\begin{algorithmic}[1]\label{alg-2}
		\REQUIRE $ \theta_{i}^{\mu}\!(\cdot),\theta_{i}^{Q}\!(\cdot),\theta_{i}^{\mu'}\!(\cdot)  $, and $ \theta_{i}^{Q'}\!(\cdot)  $ are $ c_{i} $'s DDPG model parameters; $ l_{i}^{Cri} $ and $ l_{i}^{Act} $ are the learning rates for {\it Critic} and {\it Actor} in the MainNet; $ \beta_{i} $ is the tiny updating rate for the TargetNet.
		\ENSURE $ \left[\alpha_{i}(t),\eta_{i}(t)\vert t=1,\!\cdots\!,T\right] $.
		\FOR{$ i=1,\cdots,N $}
		\STATE Initialize $ \theta_{i}^{\mu}(\cdot) $, $ \theta_{i}^{Q}(\cdot) $, $ \theta_{i}^{\mu'}(\cdot) $, and $ \theta_{i}^{Q'}(\cdot) $;
		\FOR{$ t=1,\!\cdots\!,T $}
		\STATE Observe $ s_{i}(t) $, and hence calculate $ r_{i}(t) $;
		\STATE Sample $ \left(\!s_{i}(\xi),\!a_{\mathit{CS}}(\xi),\!r_{i}(\xi),\!s_{i}(\xi+1) \right)\!,\! \xi\!=\!1,\!\cdots\!,\vert B\vert  $;
		\FOR{$ \xi=1,\!\cdots\!,\vert B\vert $}
		\STATE Calculate $ y_{i}(\xi) $ like Line 6, Algorithm \ref{alg-1};
		\ENDFOR
		\STATE Calculate\;$ \mathit{Loss}_{\mathit{CS}}(t\!-\!1) $ like Line 8, Algorithm 1;
		\STATE Update $ \theta_{i}^{Q}(t),\! \theta_{i}^{\mu}(t),\! \theta_{i}^{\mu'}(t) $,\;and $ \theta_{i}^{Q'}(t) $ like Line 9, Algorithm \ref{alg-1};
		\ENDFOR
		\ENDFOR
		\RETURN $ \mu_{i}^{*}(\cdot) $;
		\RETURN $ G^{*}_{i}(1)\!=\!A^{*}_{i}\!=\!\left[\alpha_{i}(t),\eta_{i}(t)\vert t\!=\!1,\!\cdots\!,T\right] $.
	\end{algorithmic}		
\end{algorithm}

\subsection{Workflow of PAGE}

Consequently, we propose PAGE, where $ \mathit{CS} $ and $ c_{i} $ collaboratively train global and local models by adaptively adjusting aggregation weights and local training hyper-parameters. To provide an overall insight, we illustrate the $ t $-th round of PAGE in Figure \ref{fig-1}, and depict the details as follows:
\begin{itemize}
	\item[\ding{172}]{At the beginning of the $ t $-th training round, $ \mathit{CS} $ first distributes the global model $ W(t) $ to every $ c_{i} $.}
	\item[\ding{173}]{Every $ c_{i} $ initializes the local model $ w_{i}(t) $ as $ W(t) $. Then, $ c_{i} $ updates the local DDPG model parameters $ \theta_{i}^{Q}(t) $, $ \theta_{i}^{\mu}(t) $, $ \theta_{i}^{\mu'}(t)$, and $ \theta_{i}^{Q'}(t) $ to generate $ \alpha_{i}(t) $ and $ \eta_{i}(t) $.}
	\item[\ding{174}]{$ c_{i} $ updates $ w_{i}(t) $ to $ \hat{w}_{i}(t) $ using $ \alpha_{i}(t) $ and $ \eta_{i}(t) $, simply expressed as $ \hat{w}_{i}(t)=\mathit{Train}\left(\eta_{i}(t),\alpha_{i}(t);w_{i}(t)\right) $.}
	\item[\ding{175}]{$ c_{i} $ uploads $ \hat{w}_{i}(t) $ to $ \mathit{CS} $.}
	\item[\ding{176}]{$ \mathit{CS} $ updates the global DDPG model parameters $ \theta_{\mathit{CS}}^{Q}(t) $, $ \theta_{\mathit{CS}}^{\mu}(t) $, $ \theta_{\mathit{CS}}^{\mu'}(t)$, and $ \theta_{\mathit{CS}}^{Q'}(t) $ to generate aggregation weights $ \left\{p_{i}(t)\vert i=1,\!\cdots\!,N \right\}$.}
	\item[\ding{177}]{$ \mathit{CS} $ aggregates $ \hat{w}_{i}(t),i\!=\!1,\!\cdots\!,N $ to update the global model as $ W(t\!+\!1) $ according to Eq. (\ref{eq-5}).}
	\item[\textbf{PS}:]{$ \bullet $ $ \mathit{CS} $ and $ c_{i} $ periodically perform \ding{172}-\ding{177} until $ W(t) $ and $ w_{i}(t) $ stop evolving, i.e., achieving the equilibrium of $ \textbf{P}^{\bm\prime}_{\bm 0} $.\\
		$ \bullet $ In the initial training round, the local model training and aggregation are performed by randomly selecting $ \alpha_{i}(t) $, $ \eta_{i}(t) $, and $ p_{i}(t) $, as DDPG models cannot update without prior experience \cite{Arulkumaran2017}.}
\end{itemize}

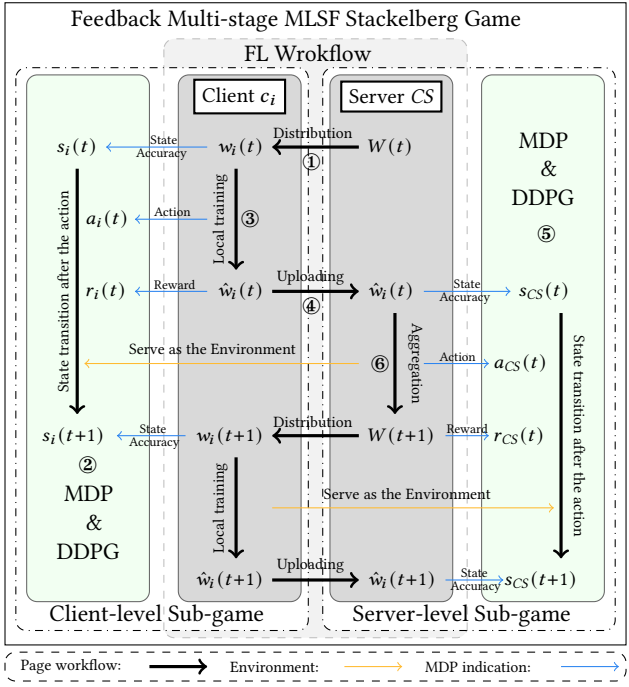
\begin{figure}
	\centering 
	\begin{tikzpicture}[scale=0.96]
		
		\path[fill=green!0, draw=black!100](0.0,0.0) rectangle (8.6,-8.9);
		
		\path[fill=black!5,rounded corners, draw=black!25, dashed](2.2,-0.5) rectangle (6.4,-8.8);
		
		\path[fill=green!0,rounded corners, draw=black!100, dash dot](0.15,-0.9) rectangle (4.2,-8.7);
		
		\path[fill=green!0,rounded corners, draw=black!100, dash dot](4.4,-0.9) rectangle (8.45,-8.7);
		
		
		\draw[dashed,-,color=black!25](2.2,-0.9)-- (2.2,-8.7);
		
		\draw[dashed,-,color=black!25](6.4,-0.9)-- (6.4,-8.7);
		
		
		\path[fill=green!5,rounded corners, draw=black!60](0.3,-1.0) rectangle (2,-8.3);
		
		\path[fill=black!15,rounded corners, draw=black!60](2.4,-1.0) rectangle (4.1,-8.3);
		
		\path[fill=black!15,rounded corners, draw=black!60](4.5,-1.0) rectangle (6.2,-8.3);
		
		\path[fill=green!5,rounded corners, draw=black!60](6.6,-1.0) rectangle (8.3,-8.3);
		
		\node[fill=blue!0,draw=black,thick] at (3.25,-1.3) (User){\small{Client $ c_{i} $}};
		
		\node[fill=blue!0,draw=black,thick] at (5.35,-1.3) (CS){\small{Server $ \mathit{CS} $}};
		
		\path[fill=yellow!0,rounded corners, draw=black!100, dashed](0.0,-9.0) rectangle (8.6,-9.4);
		
		\node[right] at (0.6,-2){\footnotesize{$ s_{i}(t) $}};
		
		\node[right] at (1.0,-3){\footnotesize{$ a_{i}(t) $}};
		
		\node[right] at (1.0,-4){\footnotesize{$ r_{i}(t) $}};
		
		\node[right] at (0.4,-6){\footnotesize{$ s_{i}(t\!+\!1) $}};
		
		\draw[very thick,->](1.0,-2.3)-- (1.0,-5.7);
		
		\node[left] at (1.0,-4){\scriptsize{\rotatebox{90}{State transition after the action}}};
		
		\draw[very thick,<-](3.7,-2)-- (4.9,-2);
		
		\node[right] at (3.6,-1.8){\scriptsize{Distribution}};
		
		\node[left] at (3.7,-2){\footnotesize{$ w_{i}(t) $}};
		
		\node[right] at (4.9,-2){\footnotesize{$ W(t) $}};
		
		\node[left] at (3.2,-3){\scriptsize{\rotatebox{90}{Local training}}};
		
		\draw[very thick,->](3.2,-2.3)-- (3.2,-3.7);
		
		\draw[very thick,->](3.7,-4)-- (4.9,-4);
		
		\node[right] at (3.65,-3.8){\scriptsize{Uploading}};
		
		\node[left] at (3.7,-4){\footnotesize{$ \hat{w}_{i}(t) $}};
		
		\node[right] at (4.9,-4){\footnotesize{$ \hat{w}_{i}(t) $}};
		
		\node[left] at (5.9,-5){\scriptsize{\rotatebox{270}{Aggregation}}};
		
		\draw[very thick,->](5.4,-4.3)-- (5.4,-5.7);
		
		\draw[very thick,<-](3.7,-6)-- (4.9,-6);
		
		\node[right] at (3.6,-5.8){\scriptsize{Distribution}};
		
		\node[left] at (3.7,-6){\footnotesize{$ w_{i}(t\!+\!1) $}};
		
		\node[right] at (4.9,-6){\footnotesize{$ W(t\!+\!1) $}};
		
		\node[left] at (3.2,-7){\scriptsize{\rotatebox{90}{Local training}}};
		
		\draw[very thick,->](3.2,-6.3)-- (3.2,-7.7);
		
		\draw[very thick,->](3.7,-8)-- (4.9,-8);
		
		\node[right] at (3.65,-7.8){\scriptsize{Uploading}};
		
		\node[left] at (3.7,-8){\footnotesize{$ \hat{w}_{i}(t\!+\!1) $}};
		
		\node[right] at (4.9,-8){\footnotesize{$ \hat{w}_{i}(t\!+\!1) $}};
		
		\node[right] at (7.0,-4){\footnotesize{$ s_{\mathit{CS}}(t) $}};
		
		\node[right] at (6.65,-5){\footnotesize{$ a_{\mathit{CS}}(t) $}};
		
		\node[right] at (6.65,-6){\footnotesize{$ r_{\mathit{CS}}(t) $}};
		
		\node[right] at (6.8,-8){\footnotesize{$ s_{\mathit{CS}}(t\!+\!1) $}};
		
		\draw[very thick,->](7.7,-4.3)-- (7.7,-7.7);
		
		\node[left] at (8.15,-6){\scriptsize{\rotatebox{270}{State transition after the action}}};

		\draw[color=GoogleBlue,thin,<-](1.4,-2)-- (2.8,-2);
		
		\node[right] at (1.9,-1.9){\tiny{State}};
		
		\node[right] at (1.7,-2.1){\tiny{Accuracy}};
		
		\draw[color=GoogleBlue,thin,<-](1.8,-3)-- (2.8,-3);
		
		\node[right] at (1.95,-2.9){\tiny{Action}};
		
		\draw[color=GoogleBlue,thin,<-](1.8,-4)-- (2.8,-4);
		
		\node[right] at (1.95,-3.9){\tiny{Reward}};
		
		\draw[color=GoogleBlue,thin,<-](1.55,-6)-- (2.5,-6);
		
		\node[right] at (1.75,-5.9){\tiny{State}};
		
		\node[right] at (1.6,-6.1){\tiny{Accuracy}};

		\draw[color=GoogleBlue,thin,->](5.8,-4)-- (7.0,-4);
		
		\node[right] at (6.1,-3.9){\tiny{State}};
		
		\node[right] at (5.9,-4.1){\tiny{Accuracy}};
		
		\draw[color=GoogleBlue,thin,->](5.8,-5)-- (6.7,-5);
		
		\node[right] at (5.9,-4.9){\tiny{Action}};
		
		\draw[color=GoogleBlue,thin,->](6.1,-6)-- (6.7,-6);
		
		\node[right] at (6.0,-5.9){\tiny{Reward}};
		
		\draw[color=GoogleBlue,thin,->](6.1,-8)-- (6.9,-8);
		
		\node[right] at (6.2,-7.9){\tiny{State}};
		
		\node[right] at (6.05,-8.1){\tiny{Accuracy}};
		
		
		\draw[color=GoogleYellow,thin,<-](1.1,-5)-- (4.9,-5);
		
		\node[right] at (1.6,-4.8){\scriptsize{Serve as the Environment}};
		
		\draw[color=GoogleYellow,thin,->](3.7,-7)-- (7.6,-7);
		
		\node[right] at (4.3,-6.8){\scriptsize{Serve as the Environment}};
		
		\node[right] at (3.2,-0.7){\normalsize{FL Wrokflow}};
		
		\node[right] at (0.8,-0.25){\normalsize{Feedback Multi-stage MLSF Stackelberg Game}};
		
		\node[right] at (0.5,-8.5){\normalsize{Client-level Sub-game}};
		
		\node[right] at (4.7,-8.5){\normalsize{Server-level Sub-game}};
		
		\node[right] at (0.7,-6.8){\normalsize{MDP}};
		
		\node[right] at (1.0,-7.2){\normalsize{\&}};
		
		\node[right] at (0.6,-7.6){\normalsize{DDPG}};
		
		\node[right] at (7.0,-1.9){\normalsize{MDP}};
		
		\node[right] at (7.3,-2.3){\normalsize{\&}};
		
		\node[right] at (6.9,-2.7){\normalsize{DDPG}};
		
		\draw[color=black,very thick,->](2.0,-9.2)-- (2.8,-9.2);
		
		\node[right] at (0.1,-9.2){\scriptsize{Page workflow:}};
		
		\draw[color=GoogleYellow,thin,->](4.7,-9.2)-- (5.5,-9.2);
		
		\node[right] at (3.0,-9.2){\scriptsize{Environment:}};
		
		\draw[color=GoogleBlue,thin,->](7.7,-9.2)-- (8.5,-9.2);
		
		\node[right] at (5.7,-9.2){\scriptsize{MDP indication:}};
		
		\node[right] at (4.0,-2.2){\ding{172}};
		
		\node[right] at (0.9,-6.4){\ding{173}};
		
		
		\node[right] at (3.15,-3.0){\ding{174}};
		
		\node[right] at (4.0,-4.2){\ding{175}};
		
		\node[right] at (7.25,-3.2){\ding{176}};
		
		\node[right] at (4.95,-5.0){\ding{177}};

	\end{tikzpicture}
	\caption{Workflow of PAGE.}
	\label{fig-1}
\end{figure}

\subsection{Theoretical Analysis for the Equilibrium}

We then analyze the existence of the equilibrium of $ \textbf{P}^{\bm\prime}_{\bm 0} $, which is equivalent to the convergence analysis in TFL and PFL. 

Before proceeding further, we define the equilibrium of $ \textbf{P}^{\bm\prime}_{\bm 0} $ in advance. Since the pay-off of every participant in a multi-stage game is an accumulating pursuit, rather than any attained peak, however large, we primarily define the pay-off functions of the sub-games.
\begin{definition}[Follower's pay-off function \cite{GameTheory1998}]\label{def-2A}
	The pay-off function of $ \mathit{CS} $ is the discounted accumulation of $ u_{\mathit{CS}}(t) $ from the $ \tau $-th stage, denoted by $ U_{\mathit{CS}}(\cdot) $ and defined as:
	\begin{equation}\label{eq-6A}
		U_{\mathit{CS}}\left( G_{\mathit{CS}}(\tau)\right)\!=\! \sum_{t=\tau}^{T}\!\left( \gamma^{t}\! \!\cdot\! u_{\mathit{CS}}(t) \right) \!=\!\sum_{t=\tau}^{T}\frac{\gamma^{t}}{F\left(W(t+1)\right)},
	\end{equation}
	where $ G_{\mathit{CS}}(\tau)\!=\!\left[g_{\mathit{CS}}(\tau),\!\cdots\!,g_{\mathit{CS}}(T)\right], \forall \tau\!=\!1,\!\cdots\!,T $ is $ \mathit{CS} $'s strategy sequence from the $ \tau $-th stage.
\end{definition}

\begin{definition}[Leader's pay-off function \cite{GameTheory1998}]\label{def-3A}
	The pay-off function of $ c_{i} $ is the discounted accumulation of $ u_{i}(t) $ from the $ \tau $-th stage, denoted by $ U_{i}(\cdot) $ and defined as:
	\begin{equation}\label{eq-7A}
		U_{i}\left(G_{i}(\tau)\right)= \sum_{t=\tau}^{T}{\left( \gamma^{t} \cdot u_{i}(t) \right) }=\sum_{t=\tau}^{T}\frac{\gamma^{t}}{f_{i}\left(w_{i}(t)\right)}.
	\end{equation}
	where $ G_{i}(\tau)=\left[g_{i}(\tau),\!\cdots\!,g_{i}(T)\right], \forall \tau=1,\!\cdots\!,T $ is $ c_{i} $'s strategy sequence from the $ \tau $-th stage.
\end{definition}

Thus, the equilibrium of $ \textbf{P}^{\bm\prime}_{\bm 0} $ can be defined as follows:
\begin{definition}[Feedback Stackelberg Equilibrium (FSE) \cite{GameTheory1998}]\label{def-4A}
	Given a feedback multi-stage MLSF Stackelberg game $ \textbf{\rm P}^{\bm\prime}_{\bm 0} $, the feedback stackelberg equilibrium is denoted by $ G^{*}(\tau)=\left[ G_{1}^{*}(\tau),\!\cdots\!,G_{N}^{*}(\tau), G_{\mathit{CS}}^{*}(\tau)\right] $ and defined as:
	\begin{equation}\label{eq-8A}
		\begin{aligned}
			&U_{\mathit{CS}}\left(G^{*}(\tau)\right) \geq U_{\mathit{CS}}\left(g_{\mathit{CS}}(\epsilon),G^{*}(\tau)\backslash g_{\mathit{CS}}^{*}(\epsilon)\right),\\
			&U_{i}\left(G^{*}(\tau)\right) \!\geq\! U_{i}\left( g_{i}(\epsilon),G^{*}(\tau)\backslash g_{i}^{*}(\epsilon)\right),\forall i\!=\!1,\!\cdots\!,N,
		\end{aligned}
	\end{equation}
	where $ \epsilon $ is the stage index in the range of $ \left[\tau,T\right] $, $ g_{i}^{*}(t) $ and $ g_{\mathit{CS}}^{*}(t) $ are optimal strategies for obtaining the maximal utilities at the $ t $-th stage, $ G_{i}^{*}(\tau)=\left[g_{i}^{*}(\tau),\!\cdots\!,g_{i}^{*}(T)\right] $ and $ G_{\mathit{CS}}^{*}(\tau)=\left[g_{\mathit{CS}}^{*}(\tau),\!\cdots\!,g_{\mathit{CS}}^{*}(T)\right] $ are the optimal strategy sequences from the $ \tau $-th stage, and $ G^{*}(\tau)\backslash g_{i}^{*}(\epsilon) $ and $ G^{*}(\tau)\backslash g_{\mathit{CS}}^{*}(\epsilon) $ indicate the optimal strategy sequences except for $ g_{i}^{*}(\epsilon) $ and $ g_{\mathit{CS}}^{*}(\epsilon) $, respectively.
\end{definition}

Definition \ref{def-4A} expounds that reaching the FSE at which $ \textbf{P}^{\bm\prime}_{\bm 0} $ ends requires a series of sequential interactions, no matter what stage the measurement starts from. 

Based on above definitions, the existence of the FSE of $ \textbf{P}^{\bm\prime}_{\bm 0} $ can be disclosed by Theorem \ref{thm-1}.

\begin{theorem}\label{thm-1}
	For $ \textbf{\rm P}^{\bm\prime}_{\bm 0} $, the feedback stackelberg equilibrium (FSE) $ G^{*}(\tau) $ always exists.
\end{theorem}

\begin{proof}
	
	We first recall the definition of the value function to measure the strategy in the FSE.
	\begin{definition}[Value Function \cite{GameTheory1998}]\label{def-5A}
		Given $ \textbf{\rm P}^{\bm\prime}_{\bm 0} $ with the FSE $ G^{*}(\tau) $, let $ Z^{*}=\left[z^{*}(\tau),\!\cdots\!,z^{*}(T)\right] $ be the associated optimal gaming condition trajectory resulting from $ z^{*}(\tau) $. Then, the value functions of $ \mathit{CS} $ and $ c_{i} $ are expressed as:
		\begin{equation}\label{eq-9A}
			V_{\mathit{CS}}^{*}\left(z^{*}(\tau)\right)=\sum_{t=\tau}^{T} {\left( \gamma^t \cdot u_{\mathit{CS}}^{*}(t)\right) },
		\end{equation}
		and
		\begin{equation}\label{eq-10A}
			V_{i}^{*}\left(z^{*}(\tau)\right)=\sum_{t=\tau}^{T} {\left( \gamma^t \cdot u^{*}_{i}(t) \right) },\forall i=1,\!\cdots\!,N,
		\end{equation}		
		where $ u_{i}^{*}(t) $ and $ u_{\mathit{CS}}^{*}(t) $ are the utilities derived from $ g_{i}^{*}(t) $ and $ g_{\mathit{CS}}^{*}(t) $, respectively.
		
	\end{definition}
	
	Thus, we can obtain $ T\!-\!\tau+1 $ sets of value functions. As a result, the only way to confirm the existence of the FSE is to verify whether these value functions satisfy the Bellman equations, shown as:
	\begin{equation}\label{eq-11A}
		V_{\mathit{CS}}^{*}(\tau) = \max_{g_{1}(t),\!\cdots\!,g_{N}(t),g_{\mathit{CS}}(t)}\!\!\!\!\!\!\!{u_{\mathit{CS}}(t)} + \gamma \cdot V_{\mathit{CS}}^{*}\left(z^{*}(\tau+1)\right),
	\end{equation}
	and
	\begin{equation}\label{eq-12A}
		\begin{aligned}
			&V_{i}^{*}(z^{*}(\tau)) = \max_{g_{1}(t),\!\cdots\!,g_{N}(t),g_{\mathit{CS}}(t)}\!\!\!\!\!\!\!{u_{i}(t)} + \gamma \cdot V_{i}^{*}\left(z^{*}(\tau+1)\right),\\
			&\forall i=1,\!\cdots\!,N.
		\end{aligned}
	\end{equation}
	Note that the first term on the right side of Eq. (\ref{eq-11A}) highlights the maximal utilities given $ Z^{*} $, and the same to Eq. (\ref{eq-12A}). As a solution, the verification could be achieved through the recursive approach, which is referred to as the {\it verification theorem} \cite{GameTheory1998}. In other words, the existence of FSE can be confirmed in specific cases for which an explicit solution of the Bellman equations can be obtained, which completes the proof.	
\end{proof}

\section{Experiments and Evaluation}
\subsection{Experimental Settings}
In this section, we compare PAGE with 10 SOTA baselines, including 5 TFLs, i.e., FedAvg \cite{McMahan2017}, FedProx \cite{Li2020FedProx}, SCAFFOLD \cite{Karimireddy2020SCAFFOLD}, FedDyn \cite{Acar2021FedDyn}, and Dap-FL \cite{Chen2023Dap-FL}, as well as five PFLs, i.e., FEDRECON \cite{Singhal2021FedRECON}, pFedMe \cite{Dinh2020pFedMe}, Ditto \cite{Li2021Ditto}, FedALA \cite{Zhang2023FedALA}, and Fed-ROD \cite{Chen2022FED-ROD}. The global model generalization and local model personalization are evaluated through the global and local model accuracy over global and local testing sets (defined below), respectively. In particular, the recorded local model accuracy is the average of local model accuracy on clients' corresponding local testing sets. Notably, all presented results are averaged over 3 runs (entire collaborative training processes) with different random seeds. 

\textbf{Datasets and models:} Our experiments are conducted on four widespread public datasets\footnote{These datasets are collected by the ML community for academic research, and no ethical considerations or legal concerns were violated.}, including Synthetic \cite{Caldas2018Synthetic}, Cifar-100 \cite{CIFAR2009}, Tiny-ImageNet \cite{TinyImageNet2015}, and Shakespeare \cite{McMahan2017}. For Synthetic, we adopt a multi-class logistic classification model with cross-entropy loss \cite{Acar2021FedDyn}. Also, we adopt ResNet-18 \cite{He2016resnet} for Cifar-100 and Tiny-ImageNet, and LSTM \cite{Hochreiter1997LSTM} for Shakespeare. More details of leveraged datasets and corresponding models are summarized in Appendix C.1.

\textbf{FL settings and data partition:} By default, our experiments involve $ 100 $ clients for the four tasks\footnote{100 is a commonly used client amount to simulate the practical FL implementation in literature. So are 50 and 1000 in the following ablation analysis.}. For {\it Logistic on Synthetic}, we use a similar data generation process in \cite{Li2020FedProx}, where each $ c_{i} $ holds $ 210 $ training samples and $ 90 $ testing samples on average, and $ \mathit{CS} $ holds $ 7500 $ testing samples. Clients' samples comprise $ 30 $ dimensions of features and $ 30 $ classes, and $ \mathit{CS} $'s samples cover all features and classes. For {\it ResNet-18 on Cifar-100} and {\it ResNet-18 on Tiny-ImageNet}, we divide the original training set into $ 100 $ parts uniformly, where the class ratio of each part follows a widely used Dirichlet distribution $ Dir(\delta\!=\!0.3) $ \cite{Tan2022}. Each part is further partitioned as the local training and testing sets on a 7:3 scale, and the original testing/validation set is assigned to $ \mathit{CS} $ as the global testing set. For {\it LSTM on Shakespeare}, we pick the role with more than $ 8000 $ sentences as the client, where $ 4900 $ and $ 2100 $ sentences are used as the local training and testing data, respectively. The remaining sentences of the pricked $ 100 $ roles are the global testing data.

\textbf{Implementation and Hyperparmeters:} All simulations are implemented on the same computing environment (Linux, 32 Intel(R) Xeon(R) Silver 4108 CPU @ 1.80GHz, NVIDIA GeForce A100, 256GB of RAM and 2T of memory) with Pytorch. In addition, the hyper-parameter settings of PAGE are summarized in Appendix C.2, and baselines are implemented with their original hyper-parameters\footnote{For datasets not involved in original baselines, we provide the appropriate hyper-parameters in our released codes.}. We release the codes and datasets at https://github.com/ivy-h7/PAGE.

\subsection{Results and Evaluation}

\textbf{Prediction accuracy comparison:} Table \ref{tab-1} illustrates the comparison between PAGE and baselines in terms of global and local model accuracy. As expected, PAGE achieves at most 39.91\% gains in terms of local model accuracy, and the global model accuracy is improved by up to 35.20\%. Surprisingly, PAGE comprehensively outperforms all baselines in most cases, where the highest global and local model accuracy is achieved simultaneously, rather than achieving a moderate balance merely. The reason behind this observation is that PAGE integrates the advantages of PFL and TFL methods, to be more specific, local fine-tuning \cite{Yu2022} and client selection \cite{Lyu2020,Wang2020}. Also, we mention that the abnormality concerning global model generalization on Synthetic is attributed to the low degree of {\it data heterogeneity}, where the global models of baselines could generalize well.

\begin{table*}[htbp]
	\caption{Prediction accuracy comparison between PAGE and baselines. We record the average and variance of global models of 3 runs, as well as the average and variance of clients' local model accuracy. Also, {\it Improvement} refers to the largest accuracy improvement. Note that {\it Logistic on Synthetic} cannot be achieved by FedRECON, as the linear layer of the logistic model cannot be partitioned to construct local variables \cite{Chen2022FED-ROD}.}
	\label{tab-1}
	\small
	\centering  
	\begin{tabular}{m{1.5cm}<{\centering}||m{1.56cm}<{\centering}|m{1.56cm}<{\centering}||m{1.56cm}<{\centering}|m{1.56cm}<{\centering}||m{1.56cm}<{\centering}|m{1.56cm}<{\centering}||m{1.56cm}<{\centering}|m{1.56cm}<{\centering}}   	
		\hline  
		\hline  
		\multicolumn{1}{c||}{\multirow{2}*{Algorithm}} & 
		\multicolumn{2}{c||}{Logistic on Synthetic} & 
		\multicolumn{2}{c||}{ResNet-18 on Cifar-100} & 
		\multicolumn{2}{c||}{ResNet-18 on Tiny-ImageNet} & 
		\multicolumn{2}{c}{LSTM on Shakespeare} \\	
		\cline{2-9} 
		&\makecell[c]{global acc (\%)} & 
		\makecell[c]{local acc (\%)} &  
		\makecell[c]{global acc (\%)} & 
		\makecell[c]{local acc (\%)} & 
		\makecell[c]{global acc (\%)} & 
		\makecell[c]{local acc (\%)} & 
		\makecell[c]{global acc (\%)} & 
		\makecell[c]{local acc (\%)} \\
		\hline
		\multicolumn{9}{c}{More attention on the comparison with local model accuracy of TFL baselines} \\
		\hline
		\makecell[c]{FedAvg} & 
		\makecell[c]{91.46$\;\pm0.07 $} & 
		\makecell[c]{95.26$\;\pm1.26 $} &
		\makecell[c]{32.97$\;\pm0.03 $} & 
		\makecell[c]{38.30$\;\pm0.44 $} & 
		\makecell[c]{7.85$\;\pm0.04 $} & 
		\makecell[c]{11.29$\;\pm0.74 $} &
		\makecell[c]{47.52$\;\pm0.07 $} & 
		\makecell[c]{40.24$\;\pm1.45 $} \\
		\hline
		\makecell[c]{FedProx} & 
		\makecell[c]{91.48$\;\pm0.05 $} & 
		\makecell[c]{95.49$\;\pm0.35 $} &
		\makecell[c]{33.46$\;\pm0.12 $} & 
		\makecell[c]{39.22$\;\pm0.77 $} & 
		\makecell[c]{7.79$\;\pm0.05 $} & 
		\makecell[c]{11.55$\;\pm0.93 $} &
		\makecell[c]{47.29$\;\pm0.12 $} & 
		\makecell[c]{40.51$\;\pm1.25 $} \\	
		\hline
		\makecell[c]{SCAFFOLD} & 
		\makecell[c]{97.37$\;\pm0.08 $} & 
		\makecell[c]{95.71$\;\pm0.53 $} &
		\makecell[c]{32.81$\;\pm0.03 $} & 
		\makecell[c]{36.12$\;\pm0.81 $} & 
		\makecell[c]{8.39$\;\pm0.02 $} & 
		\makecell[c]{9.17$\;\pm0.94 $} &
		\makecell[c]{49.14$\;\pm0.06 $} & 
		\makecell[c]{39.36$\;\pm0.49 $} \\	
		\hline
		\makecell[c]{FedDyn} & 
		\makecell[c]{\underline{97.57}$\;\pm0.07 $} & 
		\makecell[c]{94.11$\;\pm1.42 $} &
		\makecell[c]{33.47$\;\pm0.05 $} & 
		\makecell[c]{35.28$\;\pm1.11 $} & 
		\makecell[c]{7.84$\;\pm0.27 $} & 
		\makecell[c]{11.45$\;\pm1.14 $} &
		\makecell[c]{51.68$\;\pm0.14 $} & 
		\makecell[c]{42.82$\;\pm0.72 $} \\	
		\hline
		\makecell[c]{Dap-FL} & 
		\makecell[c]{92.19$\;\pm0.13 $} & 
		\makecell[c]{94.14$\;\pm1.11 $} &
		\makecell[c]{32.28$\;\pm0.27 $} & 
		\makecell[c]{40.72$\;\pm1.41 $} & 
		\makecell[c]{8.40$\;\pm0.43 $} & 
		\makecell[c]{11.75$\;\pm2.19 $} &
		\makecell[c]{51.67$\;\pm0.26 $} & 
		\makecell[c]{48.85$\;\pm1.38 $} \\	
		\hline
		\multicolumn{9}{c}{More attention on the comparison with global model accuracy of PFL baselines} \\
		\hline
		\makecell[c]{FedRECON} & 
		\makecell[c]{/} & 
		\makecell[c]{/} &
		\makecell[c]{24.88$\;\pm0.14 $} & 
		\makecell[c]{31.75$\;\pm0.65 $} & 
		\makecell[c]{6.25$\;\pm0.25 $} & 
		\makecell[c]{10.15$\;\pm1.13 $} &
		\makecell[c]{38.54$\;\pm0.06 $} & 
		\makecell[c]{35.61$\;\pm2.02 $} \\	
		\hline
		\makecell[c]{pFedMe} & 
		\makecell[c]{85.59$\;\pm0.25 $} & 
		\makecell[c]{90.23$\;\pm1.02 $} &
		\makecell[c]{30.29$\;\pm0.03 $} & 
		\makecell[c]{38.68$\;\pm0.45 $} & 
		\makecell[c]{6.60$\;\pm0.08 $} & 
		\makecell[c]{9.23$\;\pm0.36 $} &
		\makecell[c]{43.19$\;\pm0.04 $} & 
		\makecell[c]{41.99$\;\pm0.69 $} \\	
		\hline
		\makecell[c]{Ditto} & 
		\makecell[c]{92.09$\;\pm0.17 $} & 
		\makecell[c]{95.56$\;\pm1.12 $} &
		\makecell[c]{31.86$\;\pm0.24 $} & 
		\makecell[c]{39.93$\;\pm1.35 $} & 
		\makecell[c]{7.77$\;\pm0.05 $} & 
		\makecell[c]{9.59$\;\pm0.22 $} &
		\makecell[c]{48.95$\;\pm0.04 $} & 
		\makecell[c]{47.05$\;\pm0.47 $} \\	
		\hline
		\makecell[c]{FedALA} & 
		\makecell[c]{85.51$\;\pm0.04 $} & 
		\makecell[c]{95.42$\;\pm1.07 $} &
		\makecell[c]{32.10$\;\pm0.05 $} & 
		\makecell[c]{39.63$\;\pm0.84 $} & 
		\makecell[c]{7.63$\;\pm0.05 $} & 
		\makecell[c]{9.83$\;\pm0.66 $} &
		\makecell[c]{43.45$\;\pm0.09 $} & 
		\makecell[c]{46.77$\;\pm1.14 $} \\	
		\hline
		\makecell[c]{Fed-ROD} & 
		\makecell[c]{87.93$\;\pm0.21 $} & 
		\makecell[c]{90.63$\;\pm1.12 $} &
		\makecell[c]{31.75$\;\pm0.41 $} & 
		\makecell[c]{31.47$\;\pm0.59 $} & 
		\makecell[c]{8.13$\;\pm0.46 $} & 
		\makecell[c]{12.34$\;\pm0.52 $} &
		\makecell[c]{46.04$\;\pm0.11 $} & 
		\makecell[c]{43.23$\;\pm1.17 $} \\	
		\hline
		\makecell[c]{\bf PAGE} & 
		\makecell[c]{\bf 92.67$\;\pm0.13 $} & 
		\makecell[c]{\bf 96.24$\;\pm0.33 $} &
		\makecell[c]{\bf 33.55$\;\pm0.14 $} & 
		\makecell[c]{\bf 40.94$\;\pm0.26 $} & 
		\makecell[c]{\bf 8.45$\;\pm0.17 $} & 
		\makecell[c]{\bf 12.83$\;\pm0.48 $} &
		\makecell[c]{\bf 51.74$\;\pm0.24 $} & 
		\makecell[c]{\bf 49.27$\;\pm0.55 $} \\ 
		\hline 
		\makecell[c]{\it Improvement} & 
		\makecell[c]{8.37} & 
		\makecell[c]{6.66} &
		\makecell[c]{34.85} & 
		\makecell[c]{30.09} & 
		\makecell[c]{\bf 35.20} & 
		\makecell[c]{\bf 39.91} &
		\makecell[c]{34.25} & 
		\makecell[c]{38.36} \\
		\hline 
		\hline 			
	\end{tabular}
\end{table*}

\textbf{Communication efficiency comparison:} To explore the communication efficiency of PAGE, we record the convergence round in Table \ref{tab-2}. As can be observed, PAGE achieves fewer rounds in most cases, reflecting a more rapid convergence rate and higher communication efficiency. Consequently, PAGE is more competitive in MLaaS, as expensive and rare communication bandwidths are saved in the presence of satisfying the demands of customers and service providers to the greatest extent.

\begin{table*}[htbp] 
	\caption{Convergence round of PAGE and baselines. Convergence round refers to the round that the global (averaging local) model accuracy stops increasing for TFL (PFL). The column of PAGE records the round at which the FSE achieves.} 
	\label{tab-2}
	\centering  
	\begin{threeparttable}  
		\small
		\begin{tabular}{m{1.75cm}<{\centering}||m{0.9cm}<{\centering}|m{0.9cm}<{\centering}|m{0.95cm}<{\centering}|m{1.35cm}<{\centering}|m{0.95cm}<{\centering}|m{0.95cm}<{\centering}|m{1.35cm}<{\centering}|m{0.95cm}<{\centering}|m{0.95cm}<{\centering}|m{0.95cm}<{\centering}|m{1.1cm}<{\centering}} 
			
			\hline  
			\hline
			\makecell[c]{Task} &
			\makecell[c]{\bf PAGE} & 
			\makecell[c]{FedAvg} & 
			\makecell[c]{FedProx} & 
			\makecell[c]{SCAFFOLD} & 
			\makecell[c]{FedDyn} & 
			\makecell[c]{Dap-FL} & 
			\makecell[c]{FedRECON} & 
			\makecell[c]{pFedMe} & 
			\makecell[c]{Ditto} & 
			\makecell[c]{FedALA} & 
			\makecell[c]{Fed-ROD} \\
			\hline
			\makecell[c]{Synthetic} &
			\makecell[c]{891} & 
			\makecell[c]{902} & 
			\makecell[c]{896} & 
			\makecell[c]{878} & 
			\makecell[c]{901} & 
			\makecell[c]{900} & 
			\makecell[c]{/} & 
			\makecell[c]{843} & 
			\makecell[c]{491} & 
			\makecell[c]{501} & 
			\makecell[c]{497} \\
			\hline
			\makecell[c]{Cifar-100} &
			\makecell[c]{499} & 
			\makecell[c]{510} & 
			\makecell[c]{497} & 
			\makecell[c]{540} & 
			\makecell[c]{502} & 
			\makecell[c]{337} & 
			\makecell[c]{641} & 
			\makecell[c]{550} & 
			\makecell[c]{313} & 
			\makecell[c]{506} & 
			\makecell[c]{401} \\
			\hline
			\makecell[c]{Tiny-ImageNet} &
			\makecell[c]{404} & 
			\makecell[c]{430} & 
			\makecell[c]{479} & 
			\makecell[c]{422} & 
			\makecell[c]{366} & 
			\makecell[c]{402} & 
			\makecell[c]{361} & 
			\makecell[c]{513} & 
			\makecell[c]{490} & 
			\makecell[c]{523} & 
			\makecell[c]{493} \\
			\hline
			\makecell[c]{Shakespeare} &
			\makecell[c]{602} & 
			\makecell[c]{552} & 
			\makecell[c]{655} & 
			\makecell[c]{546} & 
			\makecell[c]{607} & 
			\makecell[c]{590} & 
			\makecell[c]{657} & 
			\makecell[c]{642} & 
			\makecell[c]{498} & 
			\makecell[c]{646} & 
			\makecell[c]{556} \\		 
			\hline  
			\hline 		
		\end{tabular}		
	\end{threeparttable} 
\end{table*}

\textbf{Origin of performance gains:} In Figure \ref{fig-2}, we illustrate the accuracy curves of PAGE together with the reward curves of corresponding DDPG models. One can observe the same variation trends between global/local model accuracy and server/client-side reward curves. It suggests that the server-side DDPG model facilitates global model generalization by adjusting $ p_{i} $ to obtain larger rewards, and client-side DDPG models conduct local training hyper-parameter adjustment for expected rewards, benefiting local model personalization. In the same vein, the gains of convergence rates stem from the RL-based adjustment. Besides, global and local models collaboratively evolve into stable conditions, i.e., FSE, which validates the co-opetition intention of PAGE.

\begin{figure*}[htbp]
	\centering
	\includegraphics[scale=0.27]{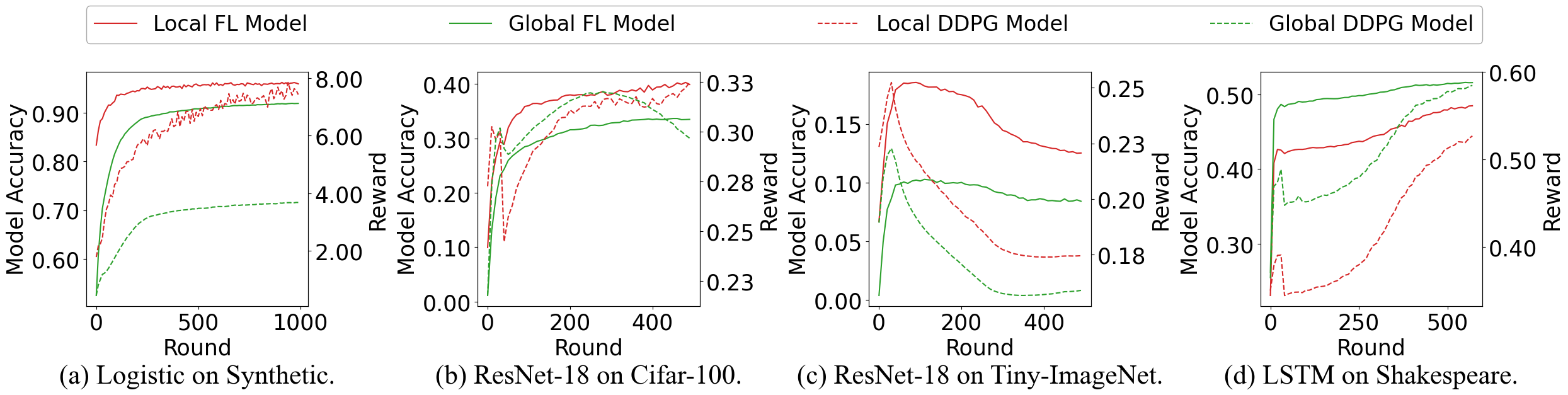}
	\caption{Model accuracy curves of PAGE together with corresponding DDPG reward curves. The left y-axes calibrate the model accuracy of FL models (solid curves), and the right y-axes calibrate the rewards of DDPG models (dotted curves).}
	\label{fig-2}
\end{figure*}

\begin{table*}[htbp]
	\caption{Comparison between PAGE and PFL under different data heterogeneity. Smaller $ \sigma $ reflects lower unbalance data distributions, and smaller $ \delta $ indicates heavier label skew.}
	\label{tab-3}
	\small
	\centering  
	\begin{tabular}{m{1.6cm}<{\centering}|m{0.9cm}<{\centering}|m{0.9cm}<{\centering}|m{0.9cm}<{\centering}|m{0.9cm}<{\centering}|m{0.9cm}<{\centering}|m{0.9cm}<{\centering}||m{0.9cm}<{\centering}|m{0.9cm}<{\centering}|m{0.9cm}<{\centering}|m{0.9cm}<{\centering}|m{0.9cm}<{\centering}|m{0.9cm}<{\centering}}   	
		\hline  
		\hline  
		\multicolumn{1}{c|}{\multirow{3}*{Algorithm}} &\multicolumn{6}{c||}{Quantity Skew -- acc (\%)} &
		\multicolumn{6}{c}{Label Skew -- acc (\%)} \\
		\cline{2-13}
		\multicolumn{1}{c|}{} & 
		\multicolumn{2}{c|}{$ \sigma=0.1 $} & 
		\multicolumn{2}{c|}{$ \sigma=0.3 $} & 
		\multicolumn{2}{c||}{$ \sigma=0.5 $} & 
		\multicolumn{2}{c|}{$ \delta=0.1 $} & 
		\multicolumn{2}{c|}{$ \delta=0.5 $} & 
		\multicolumn{2}{c}{$ \delta=1 $} \\	
		\cline{2-13} 
		&\makecell[c]{global} & 
		\makecell[c]{local} &  
		\makecell[c]{global} & 
		\makecell[c]{local} & 
		\makecell[c]{global} & 
		\makecell[c]{local} & 
		\makecell[c]{global} & 
		\makecell[c]{local} & 
		\makecell[c]{global} & 
		\makecell[c]{local} & 
		\makecell[c]{global} & 
		\makecell[c]{local} \\
		\hline
		\makecell[c]{\bf PAGE} & 
		\makecell[c]{\bf 33.47} & 
		\makecell[c]{\bf 40.96} &
		\makecell[c]{\bf 33.61} & 
		\makecell[c]{\bf 40.91} & 
		\makecell[c]{\bf 33.52} & 
		\makecell[c]{\bf 40.93} &
		\makecell[c]{\bf 32.69} & 
		\makecell[c]{\bf 54.58} & 
		\makecell[c]{\bf 33.57} & 
		\makecell[c]{\bf 40.97} &
		\makecell[c]{\bf 33.53} & 
		\makecell[c]{\bf 40.02} \\	
		\hline
		\makecell[c]{FedRECON} & 
		\makecell[c]{$ 24.13 $} & 
		\makecell[c]{$ 31.88 $} &
		\makecell[c]{$ 22.91 $} & 
		\makecell[c]{$ 32.91 $} & 
		\makecell[c]{$ 21.23 $} & 
		\makecell[c]{$ 34.95 $} &
		\makecell[c]{$ 23.77 $} & 
		\makecell[c]{$ 47.29 $} & 
		\makecell[c]{$ 25.24 $} & 
		\makecell[c]{$ 25.65 $} &
		\makecell[c]{$ 25.69 $} & 
		\makecell[c]{$ 20.13 $} \\	
		\hline
		\makecell[c]{pFedMe} & 
		\makecell[c]{$ 30.22 $} & 
		\makecell[c]{$ 38.92 $} &
		\makecell[c]{$ 30.18 $} & 
		\makecell[c]{$ 39.45 $} & 
		\makecell[c]{$ 30.09 $} & 
		\makecell[c]{$ 39.68 $} &
		\makecell[c]{$ 28.26 $} & 
		\makecell[c]{$ 47.23 $} & 
		\makecell[c]{$ 31.36 $} & 
		\makecell[c]{$ 33.28 $} &
		\makecell[c]{$ 31.47 $} & 
		\makecell[c]{$ 29.62 $} \\	
		\hline
		\makecell[c]{Ditto} & 
		\makecell[c]{$ 31.46 $} & 
		\makecell[c]{$ 39.93 $} &
		\makecell[c]{$ 30.96 $} & 
		\makecell[c]{$ 39.93 $} & 
		\makecell[c]{$ 30.31 $} & 
		\makecell[c]{$ 39.95 $} &
		\makecell[c]{$ 30.59 $} & 
		\makecell[c]{$ 53.50 $} & 
		\makecell[c]{$ 32.47 $} & 
		\makecell[c]{$ 36.46 $} &
		\makecell[c]{$ 32.79 $} & 
		\makecell[c]{$ 33.24 $} \\	
		\hline
		\makecell[c]{FedALA} & 
		\makecell[c]{$ 31.55 $} & 
		\makecell[c]{$ 39.63 $} &
		\makecell[c]{$ 31.46 $} & 
		\makecell[c]{$ 39.64 $} & 
		\makecell[c]{$ 31.08 $} & 
		\makecell[c]{$ 39.64 $} &
		\makecell[c]{$ 30.46 $} & 
		\makecell[c]{53.58} & 
		\makecell[c]{$ 32.27 $} & 
		\makecell[c]{34.98} &
		\makecell[c]{$ 32.84 $} & 
		\makecell[c]{30.24} \\	
		\hline
		\makecell[c]{Fed-ROD} & 
		\makecell[c]{$ 31.71 $} & 
		\makecell[c]{$ 32.19 $} &
		\makecell[c]{$ 31.39 $} & 
		\makecell[c]{$ 32.31 $} & 
		\makecell[c]{$ 31.06 $} & 
		\makecell[c]{$ 32.88 $} &
		\makecell[c]{$ 30.53 $} & 
		\makecell[c]{$ 49.44 $} & 
		\makecell[c]{$ 32.35 $} & 
		\makecell[c]{$ 28.18 $} &
		\makecell[c]{$ 32.16 $} & 
		\makecell[c]{$ 22.98 $} \\	
		\hline  
		\hline 			
	\end{tabular}
\end{table*}

\textbf{Performance under quantity-skewed heterogeneity:} To test the performance of PAGE facing quantity-skewed data heterogeneity, we conduct unbalanced data partitions on top of the default setting for {\it ResNet-18 on Cifar-100}, where the ratio of clients' local sample numbers follows logarithmic normal distributions\footnote{A commonly used distribution to calibrate the data quantity \cite{Zeng2023}.} with the mean of $ 0 $ and the standard deviation $ \sigma\!=\! 0.1,0.3,\text{and}\;0.5 $. In this case, we compare PAGE with PFL in the left part of Table \ref{tab-3}. As expected, the global model accuracy are higher than all PFL baselines, while keeping relatively desirable local model personalization. In particular, the global model generalization of PAGE remains stable with the increasing unbalance degree, while PFL becomes worse. Such a property is attributed to the adaptive adjustment of $ p_{i} $.

\textbf{Performance under label-skewed heterogeneity:} We then study the effectiveness of PAGE facing label-skewed data heterogeneity for {\it ResNet-18 on Cifar-100}. The right side of Table \ref{tab-3} illustrates the comparison between PAGE and PFL baselines when adjusting $ \delta $ as $ 0.1 $, $ 0.5 $, and $ 1 $ in the default setting. With the label-skewed degree increasing, PFL manifests better local model personalization, but fares less well in global model generalization, which is a somewhat disappointing property in MLaaS. Conversely, PAGE consistently exhibits outstanding personalization, while maintaining generalization. The adjustment of $ \eta_{i} $ and $ \alpha_{i} $ accounts in part for the stable performance.

\textbf{Ablation of hyper-parameter tuning completeness:} To understand how the game-based relation contributes to generalization and personalization, we conduct ablation analyses for {\it ResNet-18 on Cifar-100} by adjusting one or two factors in PAGE, while other factors remain constant. In Figure \ref{fig-3}, only adjusting $ p_{i} $ benefits the global model performance, while adjusting $ \eta_{i} $ or $ \alpha_{i} $ promotes the local model performance. By contrast, simultaneously adjusting $ \eta_{i} $ and $ \alpha_{i} $ achieves higher local model accuracy and more rapid convergence rates than solely adjusting one factor. In addition, compared to the equilibrium in the setting of remaining $ \eta_{i} $ or $ \alpha_{i} $ constant, PAGE's equilibrium has better generalization and personalization. Thus, the completeness of balance-controlling factors is confirmed.

\begin{figure}[htbp]
	\centering
	\includegraphics[scale=0.25]{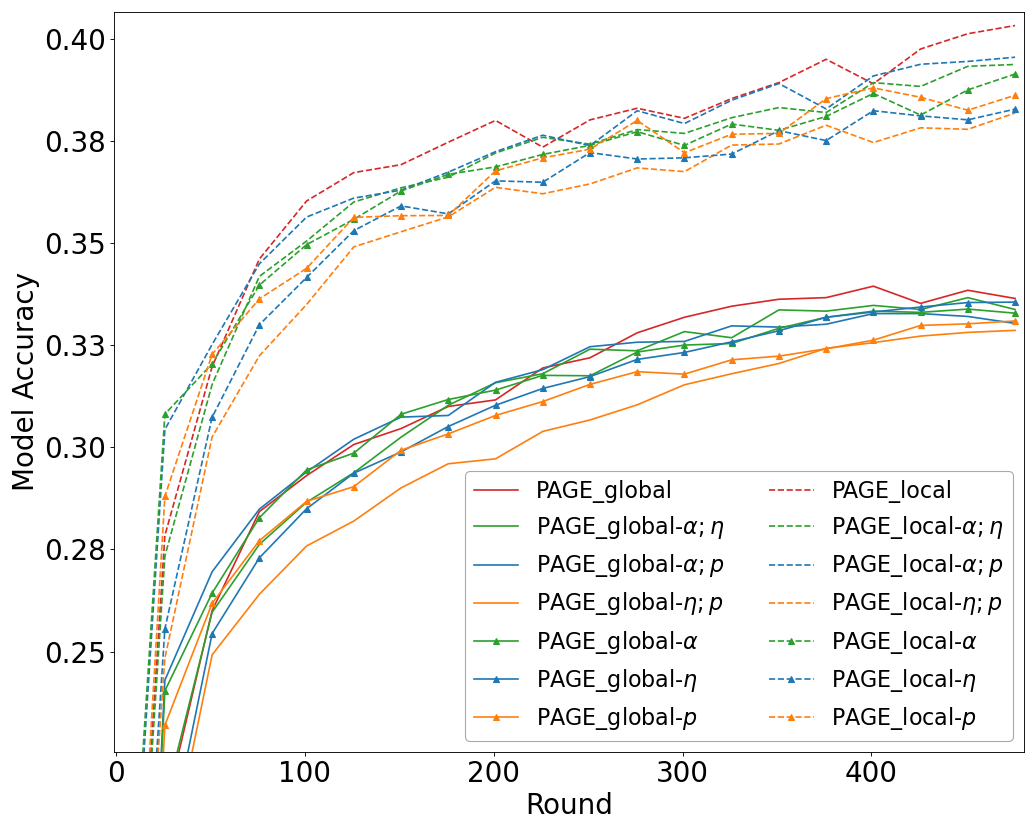}
	\caption{Completeness of balance-controlling factors.}
	\label{fig-3}
\end{figure}

\textbf{Ablation of client amount:} The top part of Table \ref{tab-4} explores the performance with different client amounts for {\it Logistic on Synthetic}. Seemingly, the balance between global and local models would not change with the client amount increasing, but requires more rounds. But we mention that the increasing round with the increasing client amount widely exists in diverse FL methods rather than merely in PAGE. The reason behind this attribute is that more participants would expand the feature space of local data, which exacerbates the difficulty of achieving the equilibrium (convergence in TFL/PFL). This provides an instructive insight into product FL with enormous clients, i.e., the practical implementation of PAGE at scale.

\begin{table}[htbp]
	\centering 
	\caption{Exploration of other properties. The last column refers to the round achieving equilibrium. $ 50 $ and $ 1000 $ indicate default settings with distinct client amounts, and $ 10\!:\!1 $ and $ 1\!:\!10 $ are the ratios between global and local rewards.}
	\label{tab-4}  
	\begin{threeparttable}  
		\small
		\begin{tabular}{m{1.6cm}<{\centering}||m{1.85cm}<{\centering}|m{1.85cm}<{\centering}|m{1.2cm}<{\centering}} 
			
			\hline  
			\hline  
			\makecell[c]{Task} & 
			\makecell[c]{global acc (\%)} & 
			\makecell[c]{local acc (\%)} & 
			\makecell[c]{Round} \\	 
			\hline  
			\multicolumn{4}{c}{\quad Client amount (Synthetic)} \\
			\hline
			\makecell[c]{PAGE-50} & 
			\makecell[c]{91.75} & 
			\makecell[c]{96.62} & 
			\makecell[c]{618} \\		 
			\hline
			\makecell[c]{PAGE-1000} & 
			\makecell[c]{92.39} & 
			\makecell[c]{96.43} & 
			\makecell[c]{928} \\
			\hline
			\makecell[c]{\bf PAGE (100)} & 
			\makecell[c]{92.67} & 
			\makecell[c]{96.24} & 
			\makecell[c]{891} \\
			\hline
			\multicolumn{4}{c}{\quad Generalization or personalization trend (Cifar-100)} \\
			\hline
			\makecell[c]{PAGE-10:1} & 
			\makecell[c]{35.11} & 
			\makecell[c]{40.19} & 
			\makecell[c]{538} \\
			\hline
			\makecell[c]{PAGE-1:10} & 
			\makecell[c]{31.91} & 
			\makecell[c]{41.55} & 
			\makecell[c]{493} \\	
			\hline
			\makecell[c]{\bf PAGE} & 
			\makecell[c]{33.55} & 
			\makecell[c]{40.94} & 
			\makecell[c]{499} \\		 
			\hline    
			\hline 		
		\end{tabular}
	\end{threeparttable} 
\end{table}

\textbf{Bias between generalization and personalization:} Also, we discuss the biased variant of PAGE for {\it ResNet-18 on Cifar-100}, where the reward ratio\footnote{The ratios are empirical settings in our simulation, which, for reproducibility, could be adjusted with the changes in the bias degree, client amount, task, etc.} between the server-side DDPG and the client-side DDPG varies to simulate the varying biases between generalization and personalization in practice. As shown in the bottom part of Table \ref{tab-4}, PAGE could tip the balance to an expected side by changing the reward ratio according to the market demand in MLaaS. Particularly, by comparing the results with baselines in Table \ref{tab-1}, the biased variants of PAGE outperform all TFL/PFL baselines in terms of corresponding global/local model accuracy and convergence rates.

\textbf{Computation efficiency:} Besides, we record the computation performance of the main operations of PAGE in Table \ref{tab-5A}, where the DDPG model training efficiency is higher than FL models by an order of magnitude. Also, the model size of the DDPG model is significantly smaller than FL models in practice, such as prevailing large language models. It suggests that PAGE is efficient in terms of computation, as the DDPG model training could be accomplished rapidly during the entire collaborative training process. Besides, DDPG can be implemented on CPU rather than rarer GPU resources, which highlights the technical feasibility of PAGE.

\begin{table}[htbp]
	\centering
	\caption{Computation performance of main operations.}
	\label{tab-5A}
	\small   
	\begin{threeparttable}  
		\begin{tabular}{m{0.8cm}<{\centering}|m{3.5cm}<{\centering}|m{2.5cm}<{\centering}} 
			
			\hline  
			\hline  
			\makecell[c]{Index} & \makecell[c]{Operation} & \makecell[c]{Time (ms/Byte)} \\ 
			\hline 
			\makecell[c]{1} & \makecell[c]{Local training} & \makecell[c]{ $1.45\!\times\!10^{-4}$}\\ 
			\hline 
			\makecell[c]{2} & \makecell[c]{Model aggregation} & \makecell[c]{$1.93\!\times\!10^{-6}$}\\  
			\hline  
			\makecell[c]{3} & \makecell[c]{Local DDPG training} & \makecell[c]{$6.84\!\times\!10^{-5}$}\\  
			\hline  
			\makecell[c]{4} & \makecell[c]{Global DDPG training} & \makecell[c]{$7.38\!\times\!10^{-5}$}\\  
			\hline  
			\hline 			
		\end{tabular}
	\end{threeparttable}  
\end{table}

\section{Conclusion and Future Work}

PAGE is the first FL algorithm that balances the local model personalization and global model generalization. A key insight into developing PAGE is that an iterative co-opetition exists between the server and clients, which runs parallel with a feedback multi-stage MLSF Stackelberg game. Particularly, the server/client-level sub-games and MDPs have uncanny resemblances. As such, PAGE introduces DDPG to solve the equilibrium of the formulated game, thereby providing a stable terminating condition for FL, i.e., the balance between personalization and generalization.

As a future work, we will take the security and privacy issues into account. In addition, by jointly considering resource heterogeneity, a variant of PAGE could be implemented in a more practical scenario, which is already investigated in Appendix D theoretically. We leave the empirical validation in the future.



\bibliographystyle{ACM-Reference-Format}
\bibliography{sample-base}


\begin{thebibliography}{46}


\ifx \showCODEN    \undefined \def \showCODEN     #1{\unskip}     \fi
\ifx \showDOI      \undefined \def \showDOI       #1{#1}\fi
\ifx \showISBNx    \undefined \def \showISBNx     #1{\unskip}     \fi
\ifx \showISBNxiii \undefined \def \showISBNxiii  #1{\unskip}     \fi
\ifx \showISSN     \undefined \def \showISSN      #1{\unskip}     \fi
\ifx \showLCCN     \undefined \def \showLCCN      #1{\unskip}     \fi
\ifx \shownote     \undefined \def \shownote      #1{#1}          \fi
\ifx \showarticletitle \undefined \def \showarticletitle #1{#1}   \fi
\ifx \showURL      \undefined \def \showURL       {\relax}        \fi
\providecommand\bibfield[2]{#2}
\providecommand\bibinfo[2]{#2}
\providecommand\natexlab[1]{#1}
\providecommand\showeprint[2][]{arXiv:#2}

\bibitem[Acar et~al\mbox{.}(2021)]%
        {Acar2021FedDyn}
\bibfield{author}{\bibinfo{person}{Durmus Alp~Emre Acar}, \bibinfo{person}{Yue
  Zhao}, \bibinfo{person}{Ramon~Matas Navarro}, \bibinfo{person}{Matthew
  Mattina}, \bibinfo{person}{Paul~N Whatmough}, {and}
  \bibinfo{person}{Venkatesh Saligrama}.} \bibinfo{year}{2021}\natexlab{}.
\newblock \showarticletitle{Federated learning based on dynamic
  regularization}. In \bibinfo{booktitle}{\emph{Proceedings of the 9th
  International Conference on Learning Representations}}.
  \bibinfo{pages}{1--36}.
\newblock


\bibitem[Altman(2021)]%
        {Altman1999}
\bibfield{author}{\bibinfo{person}{Eitan Altman}.}
  \bibinfo{year}{2021}\natexlab{}.
\newblock \bibinfo{booktitle}{\emph{Constrained Markov decision processes}}.
\newblock \bibinfo{publisher}{Routledge}, \bibinfo{address}{New York}.
\newblock


\bibitem[Arulkumaran et~al\mbox{.}(2017)]%
        {Arulkumaran2017}
\bibfield{author}{\bibinfo{person}{Kai Arulkumaran},
  \bibinfo{person}{Marc~Peter Deisenroth}, \bibinfo{person}{Miles Brundage},
  {and} \bibinfo{person}{Anil~Anthony Bharath}.}
  \bibinfo{year}{2017}\natexlab{}.
\newblock \showarticletitle{Deep Reinforcement Learning: A Brief Survey}.
\newblock \bibinfo{journal}{\emph{IEEE Signal Processing Magazine}}
  \bibinfo{volume}{34}, \bibinfo{number}{6} (\bibinfo{year}{2017}),
  \bibinfo{pages}{26--38}.
\newblock


\bibitem[Başar and Zaccour(2018)]%
        {GameTheory1998}
\bibfield{author}{\bibinfo{person}{Tamer Başar} {and} \bibinfo{person}{Georges
  Zaccour}.} \bibinfo{year}{2018}\natexlab{}.
\newblock \bibinfo{booktitle}{\emph{Handbook of Dynamic Game Theory}}.
\newblock \bibinfo{publisher}{Springer International Publishing},
  \bibinfo{address}{Cham}.
\newblock


\bibitem[Bellman(1957)]%
        {Bellman1957}
\bibfield{author}{\bibinfo{person}{Richard Bellman}.}
  \bibinfo{year}{1957}\natexlab{}.
\newblock \showarticletitle{A Markovian Decision Process}.
\newblock \bibinfo{journal}{\emph{Journal of Mathematics and Mechanics}}
  \bibinfo{volume}{6}, \bibinfo{number}{5} (\bibinfo{year}{1957}),
  \bibinfo{pages}{679--684}.
\newblock


\bibitem[Bonawitz et~al\mbox{.}(2019)]%
        {Bonawitz2019}
\bibfield{author}{\bibinfo{person}{Kallista Bonawitz}, \bibinfo{person}{Hubert
  Eichner}, \bibinfo{person}{Wolfgang Grieskamp}, \bibinfo{person}{Dzmitry
  Huba}, \bibinfo{person}{Alex Ingerman}, \bibinfo{person}{Vladimir Ivanov},
  \bibinfo{person}{Chlo{\'{e}} Kiddon}, \bibinfo{person}{Jakub
  Kone{\v{c}}n{\'y}}, \bibinfo{person}{Stefano Mazzocchi},
  \bibinfo{person}{H.~Brendan McMahan}, \bibinfo{person}{Timon~Van Overveldt},
  \bibinfo{person}{David Petrou}, \bibinfo{person}{Daniel Ramage}, {and}
  \bibinfo{person}{Jason Roselander}.} \bibinfo{year}{2019}\natexlab{}.
\newblock \showarticletitle{Towards Federated Learning at Scale: System
  Design}. In \bibinfo{booktitle}{\emph{Proceedings of the 2nd Machine Learning
  and Systems}}. \bibinfo{pages}{1: 374--388}.
\newblock


\bibitem[Boyd and Vandenberghe(2004)]%
        {Boyd2004}
\bibfield{author}{\bibinfo{person}{Stephen Boyd} {and} \bibinfo{person}{Lieven
  Vandenberghe}.} \bibinfo{year}{2004}\natexlab{}.
\newblock \bibinfo{booktitle}{\emph{Convex optimization}}.
\newblock \bibinfo{publisher}{Cambridge university press},
  \bibinfo{address}{Cambridge}.
\newblock


\bibitem[Caldas et~al\mbox{.}(2019)]%
        {Caldas2018Synthetic}
\bibfield{author}{\bibinfo{person}{Sebastian Caldas}, \bibinfo{person}{Sai
  Meher~Karthik Duddu}, \bibinfo{person}{Peter Wu}, \bibinfo{person}{Tian Li},
  \bibinfo{person}{Jakub Kone{\v{c}}n{\'y}}, \bibinfo{person}{H.~Brendan
  McMahan}, \bibinfo{person}{Virginia Smith}, {and} \bibinfo{person}{Ameet
  Talwalkar}.} \bibinfo{year}{2019}\natexlab{}.
\newblock \showarticletitle{LEAF: A Benchmark for Federated Settings}.
\newblock \bibinfo{journal}{\emph{arXiv preprint arXiv:1812.01097}}
  (\bibinfo{year}{2019}), \bibinfo{pages}{1--9}.
\newblock


\bibitem[Chen and Chao(2022)]%
        {Chen2022FED-ROD}
\bibfield{author}{\bibinfo{person}{Hong-You Chen} {and}
  \bibinfo{person}{Wei-Lun Chao}.} \bibinfo{year}{2022}\natexlab{}.
\newblock \showarticletitle{On Bridging Generic and Personalized Federated
  Learning for Image Classification}. In \bibinfo{booktitle}{\emph{Proceedings
  of the 10th International Conference on Learning Representations}}.
  \bibinfo{pages}{1--32}.
\newblock


\bibitem[Chen et~al\mbox{.}(2023)]%
        {Chen2023Dap-FL}
\bibfield{author}{\bibinfo{person}{Qian Chen}, \bibinfo{person}{Zilong Wang},
  \bibinfo{person}{Jiawei Chen}, \bibinfo{person}{Haonan Yan}, {and}
  \bibinfo{person}{Xiaodong Lin}.} \bibinfo{year}{2023}\natexlab{}.
\newblock \showarticletitle{Dap-FL: Federated Learning Flourishes by Adaptive
  Tuning and Secure Aggregation}.
\newblock \bibinfo{journal}{\emph{IEEE Transactions on Parallel and Distributed
  Systems}} \bibinfo{volume}{34}, \bibinfo{number}{6} (\bibinfo{year}{2023}),
  \bibinfo{pages}{1923--1941}.
\newblock


\bibitem[Cho et~al\mbox{.}(2023)]%
        {Cho2023}
\bibfield{author}{\bibinfo{person}{Yae~Jee Cho}, \bibinfo{person}{Jianyu Wang},
  \bibinfo{person}{Tarun Chirvolu}, {and} \bibinfo{person}{Gauri Joshi}.}
  \bibinfo{year}{2023}\natexlab{}.
\newblock \showarticletitle{Communication-Efficient and Model-Heterogeneous
  Personalized Federated Learning via Clustered Knowledge Transfer}.
\newblock \bibinfo{journal}{\emph{IEEE Journal of Selected Topics in Signal
  Processing}} \bibinfo{volume}{17}, \bibinfo{number}{1}
  (\bibinfo{year}{2023}), \bibinfo{pages}{234--247}.
\newblock


\bibitem[Dinh et~al\mbox{.}(2020)]%
        {Dinh2020pFedMe}
\bibfield{author}{\bibinfo{person}{Canh~T. Dinh}, \bibinfo{person}{Nguyen~H.
  Tran}, {and} \bibinfo{person}{Tuan~Dung Nguyen}.}
  \bibinfo{year}{2020}\natexlab{}.
\newblock \showarticletitle{Personalized Federated Learning with Moreau
  Envelopes}. In \bibinfo{booktitle}{\emph{Advances in Neural Information
  Processing Systems}}. \bibinfo{pages}{33: 21394--21405}.
\newblock


\bibitem[Duan et~al\mbox{.}(2021)]%
        {Duan2021}
\bibfield{author}{\bibinfo{person}{Moming Duan}, \bibinfo{person}{Duo Liu},
  \bibinfo{person}{Xianzhang Chen}, \bibinfo{person}{Renping Liu},
  \bibinfo{person}{Yujuan Tan}, {and} \bibinfo{person}{Liang Liang}.}
  \bibinfo{year}{2021}\natexlab{}.
\newblock \showarticletitle{Self-Balancing Federated Learning With Global
  Imbalanced Data in Mobile Systems}.
\newblock \bibinfo{journal}{\emph{IEEE Transactions on Parallel and Distributed
  Systems}} \bibinfo{volume}{32}, \bibinfo{number}{1} (\bibinfo{year}{2021}),
  \bibinfo{pages}{59--71}.
\newblock


\bibitem[Hanzely et~al\mbox{.}(2020)]%
        {Hanzely2020}
\bibfield{author}{\bibinfo{person}{Filip Hanzely},
  \bibinfo{person}{Slavom\'{\i}r Hanzely}, \bibinfo{person}{Samuel
  Horv\'{a}th}, {and} \bibinfo{person}{Peter Richt\'{a}rik}.}
  \bibinfo{year}{2020}\natexlab{}.
\newblock \showarticletitle{Lower Bounds and Optimal Algorithms for
  Personalized Federated Learning}. In \bibinfo{booktitle}{\emph{Advances in
  Neural Information Processing Systems}}. \bibinfo{pages}{33: 2304--2315}.
\newblock


\bibitem[He et~al\mbox{.}(2016)]%
        {He2016resnet}
\bibfield{author}{\bibinfo{person}{Kaiming He}, \bibinfo{person}{Xiangyu
  Zhang}, \bibinfo{person}{Shaoqing Ren}, {and} \bibinfo{person}{Jian Sun}.}
  \bibinfo{year}{2016}\natexlab{}.
\newblock \showarticletitle{Deep Residual Learning for Image Recognition}. In
  \bibinfo{booktitle}{\emph{Proceedings of the IEEE Conference on Computer
  Vision and Pattern Recognition}}. \bibinfo{pages}{770--778}.
\newblock


\bibitem[Hochreiter and Schmidhuber(1997)]%
        {Hochreiter1997LSTM}
\bibfield{author}{\bibinfo{person}{Sepp Hochreiter} {and}
  \bibinfo{person}{Jürgen Schmidhuber}.} \bibinfo{year}{1997}\natexlab{}.
\newblock \showarticletitle{Long Short-Term Memory}.
\newblock \bibinfo{journal}{\emph{Neural Computation}} \bibinfo{volume}{9},
  \bibinfo{number}{8} (\bibinfo{year}{1997}), \bibinfo{pages}{1735--1780}.
\newblock


\bibitem[Huang et~al\mbox{.}(2023)]%
        {Huang2023}
\bibfield{author}{\bibinfo{person}{Tiansheng Huang}, \bibinfo{person}{Li Shen},
  \bibinfo{person}{Yan Sun}, \bibinfo{person}{Weiwei Lin}, {and}
  \bibinfo{person}{Dacheng Tao}.} \bibinfo{year}{2023}\natexlab{}.
\newblock \showarticletitle{Fusion of Global and Local Knowledge for
  Personalized Federated Learning}.
\newblock \bibinfo{journal}{\emph{Transactions on Machine Learning Research}}
  (\bibinfo{year}{2023}), \bibinfo{pages}{1--38}.
\newblock


\bibitem[Jiang et~al\mbox{.}(2023)]%
        {Jiang2023}
\bibfield{author}{\bibinfo{person}{Yihan Jiang}, \bibinfo{person}{Jakub
  Kone{\v{c}}n{\'y}}, \bibinfo{person}{Keith Rush}, {and}
  \bibinfo{person}{Sreeram Kannan}.} \bibinfo{year}{2023}\natexlab{}.
\newblock \showarticletitle{Improving federated learning personalization via
  model agnostic meta learning}.
\newblock \bibinfo{journal}{\emph{arXiv preprint arXiv:1909.12488}}
  (\bibinfo{year}{2023}).
\newblock


\bibitem[Karimireddy et~al\mbox{.}(2020)]%
        {Karimireddy2020SCAFFOLD}
\bibfield{author}{\bibinfo{person}{Sai~Praneeth Karimireddy},
  \bibinfo{person}{Satyen Kale}, \bibinfo{person}{Mehryar Mohri},
  \bibinfo{person}{Sashank~J. Reddi}, \bibinfo{person}{Sebastian~U. Stich},
  {and} \bibinfo{person}{Ananda~Theertha Suresh}.}
  \bibinfo{year}{2020}\natexlab{}.
\newblock \showarticletitle{SCAFFOLD: Stochastic Controlled Averaging for
  Federated Learning}. In \bibinfo{booktitle}{\emph{Proceedings of the 37th
  International Conference on Machine Learning}}. \bibinfo{pages}{119:
  5132--5143}.
\newblock


\bibitem[Kone{\v{c}}n{\'y} et~al\mbox{.}(2016)]%
        {konevcny2016}
\bibfield{author}{\bibinfo{person}{Jakub Kone{\v{c}}n{\'y}},
  \bibinfo{person}{H.~Brendan McMahan}, \bibinfo{person}{Felix~X. Yu},
  \bibinfo{person}{Peter Richt{\'a}rik}, \bibinfo{person}{Ananda~Theertha
  Suresh}, {and} \bibinfo{person}{Dave Bacon}.}
  \bibinfo{year}{2016}\natexlab{}.
\newblock \showarticletitle{Federated learning: Strategies for improving
  communication efficiency}. In \bibinfo{booktitle}{\emph{NIPS Workshop on
  Private Multi-Party Machine Learning}}. \bibinfo{pages}{1--10}.
\newblock


\bibitem[Krizhevsky(2009)]%
        {CIFAR2009}
\bibfield{author}{\bibinfo{person}{Alex Krizhevsky}.}
  \bibinfo{year}{2009}\natexlab{}.
\newblock \showarticletitle{Learning multiple layers of features from tiny
  images}.
\newblock \bibinfo{journal}{\emph{Technical Report}} (\bibinfo{year}{2009}).
\newblock


\bibitem[Le and Yang(2015)]%
        {TinyImageNet2015}
\bibfield{author}{\bibinfo{person}{Ya Le} {and} \bibinfo{person}{Xuan Yang}.}
  \bibinfo{year}{2015}\natexlab{}.
\newblock \showarticletitle{Tiny imagenet visual recognition challenge}.
\newblock \bibinfo{journal}{\emph{CS 231N}} \bibinfo{volume}{7},
  \bibinfo{number}{7} (\bibinfo{year}{2015}), \bibinfo{pages}{1--6}.
\newblock


\bibitem[Li et~al\mbox{.}(2023)]%
        {Li2023}
\bibfield{author}{\bibinfo{person}{Guanghao Li}, \bibinfo{person}{Wansen Wu},
  \bibinfo{person}{Yan Sun}, \bibinfo{person}{Li Shen},
  \bibinfo{person}{Baoyuan Wu}, {and} \bibinfo{person}{Dacheng Tao}.}
  \bibinfo{year}{2023}\natexlab{}.
\newblock \showarticletitle{Visual Prompt Based Personalized Federated
  Learning}.
\newblock \bibinfo{journal}{\emph{arXiv preprint arXiv:2303.08678}}
  (\bibinfo{year}{2023}).
\newblock


\bibitem[Li et~al\mbox{.}(2021)]%
        {Li2021Ditto}
\bibfield{author}{\bibinfo{person}{Tian Li}, \bibinfo{person}{Shengyuan Hu},
  \bibinfo{person}{Ahmad Beirami}, {and} \bibinfo{person}{Virginia Smith}.}
  \bibinfo{year}{2021}\natexlab{}.
\newblock \showarticletitle{Ditto: Fair and Robust Federated Learning Through
  Personalization}. In \bibinfo{booktitle}{\emph{Proceedings of the 38th
  International Conference on Machine Learning}}. \bibinfo{pages}{139:
  6357--6368}.
\newblock


\bibitem[Li et~al\mbox{.}(2020b)]%
        {Li2020FedProx}
\bibfield{author}{\bibinfo{person}{Tian Li}, \bibinfo{person}{Anit~Kumar Sahu},
  \bibinfo{person}{Manzil Zaheer}, \bibinfo{person}{Maziar Sanjabi},
  \bibinfo{person}{Ameet Talwalkar}, {and} \bibinfo{person}{Virginia Smith}.}
  \bibinfo{year}{2020}\natexlab{b}.
\newblock \showarticletitle{Federated Optimization in Heterogeneous Networks}.
  In \bibinfo{booktitle}{\emph{Proceedings of the 3rd Machine Learning and
  Systems}}. \bibinfo{pages}{2: 429--450}.
\newblock


\bibitem[Li et~al\mbox{.}(2020a)]%
        {Li2020}
\bibfield{author}{\bibinfo{person}{Xiang Li}, \bibinfo{person}{Kaixuan Huang},
  \bibinfo{person}{Wenhao Yang}, \bibinfo{person}{Shusen Wang}, {and}
  \bibinfo{person}{Zhihua Zhang}.} \bibinfo{year}{2020}\natexlab{a}.
\newblock \showarticletitle{On the convergence of fedavg on non-iid data}. In
  \bibinfo{booktitle}{\emph{Proceedings of the 8th International Conference on
  Learning Representations}}. \bibinfo{pages}{1--26}.
\newblock


\bibitem[Lillicrap et~al\mbox{.}(2016)]%
        {Lillicrap2016}
\bibfield{author}{\bibinfo{person}{Timothy~P. Lillicrap},
  \bibinfo{person}{Jonathan~J. Hunt}, \bibinfo{person}{Alexander Pritzel},
  \bibinfo{person}{Nicolas Heess}, \bibinfo{person}{Tom Erez},
  \bibinfo{person}{Yuval Tassa}, \bibinfo{person}{David Silver}, {and}
  \bibinfo{person}{Daan Wierstra}.} \bibinfo{year}{2016}\natexlab{}.
\newblock \showarticletitle{Continuous control with deep reinforcement
  learning}. In \bibinfo{booktitle}{\emph{Proceedings of the 4th International
  Conference on Learning Representations}}. \bibinfo{pages}{1--14}.
\newblock


\bibitem[Liu et~al\mbox{.}(2022)]%
        {Liu2023}
\bibfield{author}{\bibinfo{person}{Zelei Liu}, \bibinfo{person}{Yuanyuan Chen},
  \bibinfo{person}{Yansong Zhao}, \bibinfo{person}{Han Yu},
  \bibinfo{person}{Yang Liu}, \bibinfo{person}{Renyi Bao},
  \bibinfo{person}{Jinpeng Jiang}, \bibinfo{person}{Zaiqing Nie},
  \bibinfo{person}{Qian Xu}, {and} \bibinfo{person}{Qiang Yang}.}
  \bibinfo{year}{2022}\natexlab{}.
\newblock \showarticletitle{Contribution-Aware Federated Learning for Smart
  Healthcare}. In \bibinfo{booktitle}{\emph{Proceedings of the AAAI Conference
  on Artificial Intelligence}}. \bibinfo{pages}{36(11): 12396--12404}.
\newblock


\bibitem[Lyu et~al\mbox{.}(2020)]%
        {Lyu2020}
\bibfield{author}{\bibinfo{person}{Lingjuan Lyu}, \bibinfo{person}{Jiangshan
  Yu}, \bibinfo{person}{Karthik Nandakumar}, \bibinfo{person}{Yitong Li},
  \bibinfo{person}{Xingjun Ma}, \bibinfo{person}{Jiong Jin},
  \bibinfo{person}{Han Yu}, {and} \bibinfo{person}{Kee~Siong Ng}.}
  \bibinfo{year}{2020}\natexlab{}.
\newblock \showarticletitle{Towards Fair and Privacy-Preserving Federated Deep
  Models}.
\newblock \bibinfo{journal}{\emph{IEEE Transactions on Parallel and Distributed
  Systems}} \bibinfo{volume}{31}, \bibinfo{number}{11} (\bibinfo{year}{2020}),
  \bibinfo{pages}{2524--2541}.
\newblock


\bibitem[Marfoq et~al\mbox{.}(2022)]%
        {Marfoq2022}
\bibfield{author}{\bibinfo{person}{Othmane Marfoq}, \bibinfo{person}{Giovanni
  Neglia}, \bibinfo{person}{Laetitia Kameni}, {and} \bibinfo{person}{Richard
  Vidal}.} \bibinfo{year}{2022}\natexlab{}.
\newblock \showarticletitle{Personalized Federated Learning through Local
  Memorization}. In \bibinfo{booktitle}{\emph{Proceedings of the 39th
  International Conference on Machine Learning}}. \bibinfo{pages}{162:
  15070--15092}.
\newblock


\bibitem[McMahan et~al\mbox{.}(2017)]%
        {McMahan2017}
\bibfield{author}{\bibinfo{person}{H.~Brendan McMahan}, \bibinfo{person}{Eider
  Moore}, \bibinfo{person}{Daniel Ramage}, \bibinfo{person}{Seth Hampson},
  {and} \bibinfo{person}{Blaise Aguera~y Arcas}.}
  \bibinfo{year}{2017}\natexlab{}.
\newblock \showarticletitle{Communication-Efficient Learning of Deep Networks
  from Decentralized Data}. In \bibinfo{booktitle}{\emph{Proceedings of the
  20th International Conference on Artificial Intelligence and Statistics}}.
  \bibinfo{pages}{54: 1273--1282}.
\newblock


\bibitem[Niu et~al\mbox{.}(2020)]%
        {Niu2020}
\bibfield{author}{\bibinfo{person}{Chaoyue Niu}, \bibinfo{person}{Fan Wu},
  \bibinfo{person}{Shaojie Tang}, \bibinfo{person}{Lifeng Hua},
  \bibinfo{person}{Rongfei Jia}, \bibinfo{person}{Chengfei Lv},
  \bibinfo{person}{Zhihua Wu}, {and} \bibinfo{person}{Guihai Chen}.}
  \bibinfo{year}{2020}\natexlab{}.
\newblock \showarticletitle{Billion-Scale Federated Learning on Mobile Clients:
  A Submodel Design with Tunable Privacy}. In
  \bibinfo{booktitle}{\emph{Proceedings of the 26th Annual International
  Conference on Mobile Computing and Networking}}. \bibinfo{pages}{31:
  405--418}.
\newblock


\bibitem[Singhal et~al\mbox{.}(2021)]%
        {Singhal2021FedRECON}
\bibfield{author}{\bibinfo{person}{Karan Singhal}, \bibinfo{person}{Hakim
  Sidahmed}, \bibinfo{person}{Zachary Garrett}, \bibinfo{person}{Shanshan Wu},
  \bibinfo{person}{John Rush}, {and} \bibinfo{person}{Sushant Prakash}.}
  \bibinfo{year}{2021}\natexlab{}.
\newblock \showarticletitle{Federated Reconstruction: Partially Local Federated
  Learning}. In \bibinfo{booktitle}{\emph{Advances in Neural Information
  Processing Systems}}. \bibinfo{pages}{34: 11220--11232}.
\newblock


\bibitem[Slater(2013)]%
        {Slater2014}
\bibfield{author}{\bibinfo{person}{Morton Slater}.}
  \bibinfo{year}{2013}\natexlab{}.
\newblock \showarticletitle{Lagrange Multipliers Revisited}.
\newblock \bibinfo{journal}{\emph{Traces and Emergence of Nonlinear
  Programming}} (\bibinfo{year}{2013}), \bibinfo{pages}{293--306}.
\newblock


\bibitem[Sun et~al\mbox{.}(2021)]%
        {Sun2021}
\bibfield{author}{\bibinfo{person}{Benyuan Sun}, \bibinfo{person}{Hongxing
  Huo}, \bibinfo{person}{Yi Yang}, {and} \bibinfo{person}{Bo Bai}.}
  \bibinfo{year}{2021}\natexlab{}.
\newblock \showarticletitle{PartialFed: Cross-Domain Personalized Federated
  Learning via Partial Initialization}. In \bibinfo{booktitle}{\emph{Advances
  in Neural Information Processing Systems}}. \bibinfo{pages}{34:
  23309--23320}.
\newblock


\bibitem[Tan et~al\mbox{.}(2022)]%
        {Tan2022}
\bibfield{author}{\bibinfo{person}{Alysa~Ziying Tan}, \bibinfo{person}{Han Yu},
  \bibinfo{person}{Lizhen Cui}, {and} \bibinfo{person}{Qiang Yang}.}
  \bibinfo{year}{2022}\natexlab{}.
\newblock \showarticletitle{Towards Personalized Federated Learning}.
\newblock \bibinfo{journal}{\emph{IEEE Transactions on Neural Networks and
  Learning Systems}}  \bibinfo{volume}{Early access} (\bibinfo{year}{2022}),
  \bibinfo{pages}{1--17}.
\newblock


\bibitem[Wang et~al\mbox{.}(2020)]%
        {Wang2020}
\bibfield{author}{\bibinfo{person}{Hao Wang}, \bibinfo{person}{Zakhary Kaplan},
  \bibinfo{person}{Di Niu}, {and} \bibinfo{person}{Baochun Li}.}
  \bibinfo{year}{2020}\natexlab{}.
\newblock \showarticletitle{Optimizing Federated Learning on Non-IID Data with
  Reinforcement Learning}. In \bibinfo{booktitle}{\emph{Proceedings of the IEEE
  Conference on Computer Communications}}. \bibinfo{pages}{1698--1707}.
\newblock


\bibitem[Wang et~al\mbox{.}(2023)]%
        {Yang2023}
\bibfield{author}{\bibinfo{person}{Kaibin Wang}, \bibinfo{person}{Qiang He},
  \bibinfo{person}{Feifei Chen}, \bibinfo{person}{Chunyang Chen},
  \bibinfo{person}{Faliang Huang}, \bibinfo{person}{Hai Jin}, {and}
  \bibinfo{person}{Yun Yang}.} \bibinfo{year}{2023}\natexlab{}.
\newblock \showarticletitle{FlexiFed: Personalized Federated Learning for Edge
  Clients with Heterogeneous Model Architectures}. In
  \bibinfo{booktitle}{\emph{Proceedings of the ACM Web Conference}}.
  \bibinfo{pages}{2979–2990}.
\newblock


\bibitem[Wang et~al\mbox{.}(2019)]%
        {Wang2019}
\bibfield{author}{\bibinfo{person}{Kangkang Wang}, \bibinfo{person}{Rajiv
  Mathews}, \bibinfo{person}{Chlo{\'e} Kiddon}, \bibinfo{person}{Hubert
  Eichner}, \bibinfo{person}{Fran{\c{c}}oise Beaufays}, {and}
  \bibinfo{person}{Daniel Ramage}.} \bibinfo{year}{2019}\natexlab{}.
\newblock \showarticletitle{Federated evaluation of on-device personalization}.
\newblock \bibinfo{journal}{\emph{arXiv preprint arXiv:1910.10252}}
  (\bibinfo{year}{2019}).
\newblock


\bibitem[Wu and Wang(2021)]%
        {Wu2021}
\bibfield{author}{\bibinfo{person}{Hongda Wu} {and} \bibinfo{person}{Ping
  Wang}.} \bibinfo{year}{2021}\natexlab{}.
\newblock \showarticletitle{Fast-Convergent Federated Learning With Adaptive
  Weighting}.
\newblock \bibinfo{journal}{\emph{IEEE Transactions on Cognitive Communications
  and Networking}} \bibinfo{volume}{7}, \bibinfo{number}{4}
  (\bibinfo{year}{2021}), \bibinfo{pages}{1078--1088}.
\newblock


\bibitem[You et~al\mbox{.}(2023)]%
        {You2023}
\bibfield{author}{\bibinfo{person}{Chaoqun You}, \bibinfo{person}{Kun Guo},
  \bibinfo{person}{Howard~H. Yang}, {and} \bibinfo{person}{Tony Q.~S. Quek}.}
  \bibinfo{year}{2023}\natexlab{}.
\newblock \showarticletitle{Hierarchical Personalized Federated Learning Over
  Massive Mobile Edge Computing Networks}.
\newblock \bibinfo{journal}{\emph{IEEE Transactions on Wireless
  Communications}}  \bibinfo{volume}{Early access} (\bibinfo{year}{2023}),
  \bibinfo{pages}{1--17}.
\newblock


\bibitem[Yu et~al\mbox{.}(2022)]%
        {Yu2022}
\bibfield{author}{\bibinfo{person}{Tao Yu}, \bibinfo{person}{Eugene
  Bagdasaryan}, {and} \bibinfo{person}{Vitaly Shmatikov}.}
  \bibinfo{year}{2022}\natexlab{}.
\newblock \showarticletitle{Salvaging federated learning by local adaptation}.
\newblock \bibinfo{journal}{\emph{arXiv preprint arXiv:2002.04758}}
  (\bibinfo{year}{2022}).
\newblock


\bibitem[Zeng et~al\mbox{.}(2023)]%
        {Zeng2023}
\bibfield{author}{\bibinfo{person}{Dun Zeng}, \bibinfo{person}{Siqi Liang},
  \bibinfo{person}{Xiangjing Hu}, \bibinfo{person}{Hui Wang}, {and}
  \bibinfo{person}{Zenglin Xu}.} \bibinfo{year}{2023}\natexlab{}.
\newblock \showarticletitle{FedLab: A Flexible Federated Learning Framework}.
\newblock \bibinfo{journal}{\emph{Journal of Machine Learning Research}}
  (\bibinfo{year}{2023}), \bibinfo{pages}{1--7}.
\newblock


\bibitem[Zhang et~al\mbox{.}(2021)]%
        {Zhang2021}
\bibfield{author}{\bibinfo{person}{Jie Zhang}, \bibinfo{person}{Song Guo},
  \bibinfo{person}{Xiaosong Ma}, \bibinfo{person}{Haozhao Wang},
  \bibinfo{person}{Wencao Xu}, {and} \bibinfo{person}{Feijie Wu}.}
  \bibinfo{year}{2021}\natexlab{}.
\newblock \showarticletitle{Parameterized Knowledge Transfer for Personalized
  Federated Learning}. In \bibinfo{booktitle}{\emph{Advances in Neural
  Information Processing Systems}}. \bibinfo{pages}{34: 10092--10104}.
\newblock


\bibitem[Zhang et~al\mbox{.}(2023)]%
        {Zhang2023FedALA}
\bibfield{author}{\bibinfo{person}{Jianqing Zhang}, \bibinfo{person}{Yang Hua},
  \bibinfo{person}{Hao Wang}, \bibinfo{person}{Tao Song},
  \bibinfo{person}{Zhengui Xue}, \bibinfo{person}{Ruhui Ma}, {and}
  \bibinfo{person}{Haibing Guan}.} \bibinfo{year}{2023}\natexlab{}.
\newblock \showarticletitle{FedALA: Adaptive Local Aggregation for Personalized
  Federated Learning}. In \bibinfo{booktitle}{\emph{Proceedings of the AAAI
  Conference on Artificial Intelligence}}. \bibinfo{pages}{37(9):
  11237--11244}.
\newblock


\bibitem[Zhu et~al\mbox{.}(2021)]%
        {Zhu2021}
\bibfield{author}{\bibinfo{person}{Zhuangdi Zhu}, \bibinfo{person}{Junyuan
  Hong}, {and} \bibinfo{person}{Jiayu Zhou}.} \bibinfo{year}{2021}\natexlab{}.
\newblock \showarticletitle{Data-Free Knowledge Distillation for Heterogeneous
  Federated Learning}. In \bibinfo{booktitle}{\emph{Proceedings of the 38th
  International Conference on Machine Learning}}. \bibinfo{pages}{139:
  12878--12889}.
\newblock


\end{thebibliography}


\end{document}